\documentclass[10pt,twocolumn,letterpaper]{article}

\usepackage{iccv}
\usepackage{times}
\usepackage{epsfig}
\usepackage{graphicx}
\usepackage{amsmath}
\usepackage{amssymb}
\usepackage{iccv}
\usepackage{times}
\usepackage{epsfig}
\usepackage{caption}
\usepackage{subcaption}
\usepackage{graphicx}
\usepackage{amsmath}
\usepackage{amssymb}
\usepackage{mathtools}
\usepackage{amsthm}
\usepackage[margin=2cm,heightrounded]{geometry}
\usepackage{newtxtext,newtxmath}
\usepackage[export]{adjustbox}
\usepackage{multicol}
\theoremstyle{plain}
\newtheorem{theorem}{Theorem}[section]

\newtheorem{lemma}[theorem]{Lemma}

\theoremstyle{definition}

\theoremstyle{remark}
\newcommand{\norm}[1]{\left\lVert#1\right\rVert}
\newcommand{\MUCMM}{\texttt{MUC2M}}
\newcommand{\baseresnet}{\texttt{Base ResNet50}}
\newcommand{\baseclip}{\texttt{Base CLIP-ResNet50}}
\newcommand{\uni}{\texttt{Unif.}}

\newtheorem{prop}[theorem]{Proposition}
\theoremstyle{definition}

\newcommand{\FL}{\texttt{FL}}
\newcommand{\Conv}{\texttt{Conv}}
\newcommand{\ReLU}{\texttt{ReLU}}

\newcommand{\batchtwod}{\texttt{BatchNorm2d}}
\newcommand{\batchoned}{\texttt{BatchNorm1d}}

\usepackage[pagebackref=true,breaklinks=true,letterpaper=true,colorlinks,bookmarks=false]{hyperref}

 \iccvfinalcopy 

\def\iccvPaperID{10398} 

\ificcvfinal\fi

\begin{document}

\title{Multimodal Understanding \\Through Correlation Maximization and Minimization}

\author{Yifeng Shi\\
UNC-Chapel Hill\\
{\tt\small yifengs@cs.unc.edu}
\and
Marc Niethammer\\
UNC-Chapel Hill\\
{\tt\small mn@cs.unc.edu}
}

\maketitle
\ificcvfinal\thispagestyle{empty}\fi

\begin{abstract}
Multimodal learning has mainly focused on learning large models on, and fusing feature representations from, different modalities for better performances on downstream tasks. In this work, we take a detour from this trend and study the intrinsic nature of multimodal data by asking the following questions: 1) Can we learn more structured latent representations of general multimodal data?; and 2) can we intuitively understand, both mathematically and visually, what the latent representations capture? To answer 1), we propose a general and lightweight framework, \textbf{M}ultimodal \textbf{U}nderstanding Through \textbf{C}orrelation \textbf{M}aximization and \textbf{M}inimization (\MUCMM), that can be incorporated into any large pre-trained network. \MUCMM\ learns both the common and individual representations. The common representations capture what is common between the modalities; the individual representations capture the unique aspect of the modalities. To answer 2), we propose novel scores that summarize the learned common and individual structures and visualize the score gradients with respect to the input, visually discerning what the different representations capture. We further provide mathematical intuitions of the computed gradients in a linear setting, and demonstrate the effectiveness of our approach through a variety of experiments. 
\end{abstract}

\section{Introduction}
With rapidly increasing varieties of data being collected, multimodal learning has gained prominence~\cite{Baltruaitis2017MultimodalML,Bayoudh2021ASO,Xu2022MultimodalLW}. Much of the recent work in multimodal learning focuses on the fusion of information from different modalities for a diverse set of downstream tasks, examples of which range from autonomous driving~\cite{Xiao2019MultimodalEA}, to image captioning \cite{Stefanini2021FromST}, and to recent work on image generation~\cite{Dhariwal2021DiffusionMB}. In this work, however, we focus on understanding the intrinsic nature of multimodal data beyond a particular downstream task. 

We are motivated by the question how can we understand the intrinsic nature of multimodal data without supervision? For instance, given a pair of image and text, can we highlight the object in the image that is described by the text in an unsupervised manner? Or given a pair of images, can we learn what objects the images have in common, and what objects are unique to each image without supervision? Indeed, one way to think about multimodal data is to view each modality as a combination of the common structure, i.e. the information it shares with other modalities, and the individual structure, i.e. the information that is unique to that modality. In this work, we aim to learn the common and individual structures of multimodal data in an unsupervised manner.


Two characteristics should be present to corroborate a model's ability to distinguish between the common and individual structures: 1) the latent representations learned for the two structures should be uncorrelated with each other; and 2) we can visualize the learned structures to demystify the otherwise opaque latent representations. While large pre-trained models for multimodal data exist~\cite{Bugliarello2020MultimodalPU,Bayoudh2021ASO,Xu2022MultimodalLW,Long2022VisionandLanguagePM}, their extracted feature representations do not provide a clear separation between the common and individual aspects of the respective modalities. We nevertheless want to build on such models as they are powerful feature extractors trained on large datasets, whereas training them from scratch would be computationally costly and not possible if only limited data is available~\cite{Brigato2020ACL,Zhuang2019ACS}. To this end, to learn latent representations for the common structure, we introduce a reformulated deep canonical correlation analysis (deep CCA)~\cite{Andrew2013DeepCC} that is easier to optimize; to learn that for the individual structure, we introduce constraints to ensure the representations capture different information as opposed to the common latent representations. Furthermore, we introduce contractive regularization in the style of Rifai \etal~\cite{Rifai2011ContractiveAE} to prevent erroneous learning and to improve learning efficiency. Altogether, we propose a unified, lightweight framework that can utilize any appropriate large pre-trained networks to learn latent representations for the common and individual structures.

To visualize the learned structures, the gradient map with respect to the input data is a good candidate. Such approaches have been extensively studied in supervised settings where one computes the gradient map of the model output (e.g. a classification output) with respect to the input for visualization. However, since our setting is unsupervised, we first have to define an appropriate score function with respect to which a gradient map can be computed. To this end, we introduce sensible scores of the common and individual structures based on their specific latent representations. Last but not least, we also provide theoretical insights into our proposed score functions in a linear setting.  
\vspace{-4.5mm}
\paragraph{Main contributions} Our main contributions are: 1) We introduce \MUCMM, a general lightweight framework that can be used in conjunction with existing large pre-trained networks to learn and visualize the common and individual structures of multimodal data; 2) We introduce a reformulated deep CCA objective that is easier to optimize; 3) We incorporate the contractive regularization into \MUCMM\ to prevent erroneous learning; 4) We introduce novel scores that summarize both structures for gradient computations; 5) We provide mathematical intuitions of \MUCMM\ in a linear setting.

\section{Related Work}
\paragraph{CCA and its application in self-supervised learning} CCA \cite{Hotelling1936RelationsBT,Knapp1978CanonicalCA} aims to find the relation between two sets of variables. Nonlinear extensions of CCA, specifically the deep CCA \cite{Andrew2013DeepCC,Wang2016OnDM}, were introduced to learn nonlinear relations and handle more complicated data modalities such as images and audios. The idea of seeking the common structure between different modalities is a general one, and turns out to be especially applicable in self-supervised learning. For example, Zbontar \etal \cite{Zbontar2021BarlowTS} and Bardes \etal \cite{Bardes2021VICRegVR} proposed to learn features that are robust to data augmentations by enforcing CCA-related losses between two different augmentations of the same input data. Zhang \etal \cite{Zhang2021FromCC} applied similar ideas to graph learning. Balestriero \etal \cite{Balestriero2022ContrastiveAN} provided an overview and theoretical analysis of the relations between many self-supervised methods and spectral embedding approaches. We rewrite the deep CCA objective proposed in Chang \etal \cite{Chang2017ScalableAE} to a different, but equivalent, formation that helps the model better avoid local maxima during optimization.
\vspace{-3mm}
\paragraph{Common and individual structures in multimodal data} Works that focus on finding the common and individual structures in multimodal data exist. Lock \etal \cite{Lock2011JOINTAI} and Luo \etal \cite{Luo2018ConsistentAS} assumed a linear structural decomposition and proposed an iterative algorithm to find such decomposition. Feng \etal \cite{Feng2017AnglebasedJA} used perturbation theory to find sets of bases for the subspaces that span the common and individual structures, and linearly project the data onto the subspaces. However, those approaches can only model linear relations (via either assuming linear structural decompositions or using linear projections), and cannot be easily scaled to more challenging data types such as images. Xu \etal \cite{Xu2021MultiVAELD} proposed to learn ``view-common'' and ``view-peculiar'' representations for multi-view clustering using a variational autoencoder (VAE) \cite{Kingma2013AutoEncodingVB}. Abid \cite{Abid2019ContrastiveVA} proposed a similar VAE-based model to learn salient features that are specific to one view but not the other. Since one of our main motivations is to make \MUCMM\  a lightweight model that can be incorporated with any pre-trained network (a vast majority of which are not probabilistic), we design \MUCMM\  in a non-probabilistic manner. Concurrent with our work, Sun \etal \cite{Sun2022AddressingCB} proposed a similar framework to address the contradiction between correlation maximization and the reconstruction objective in DCCAE \cite{Wang2016OnDM}. Different from \MUCMM\ , they used the original CCA objective to learn the common representations whereas we propose an equivalent, but advantageous, formulation (see Sec.~\ref{barlow_advantage}). They also learned the individual representations by maximizing the Jenson-Shannon Divergence between the individual and common representations whereas we propose to decorrelate the two representations. The choice of decorrelation lends us a more natural interpretation of the gradient maps we compute where we connect them to linear projections in the linear setting. Last but not least, as one of our main contributions, we propose novel scores that are motivated theoretically from the linear setting (see Sec.~\ref{linear_case_proof}) for visualizing the learned structures in the input space.

\begin{figure*}[t]
  \includegraphics[width=\textwidth,height=6cm]{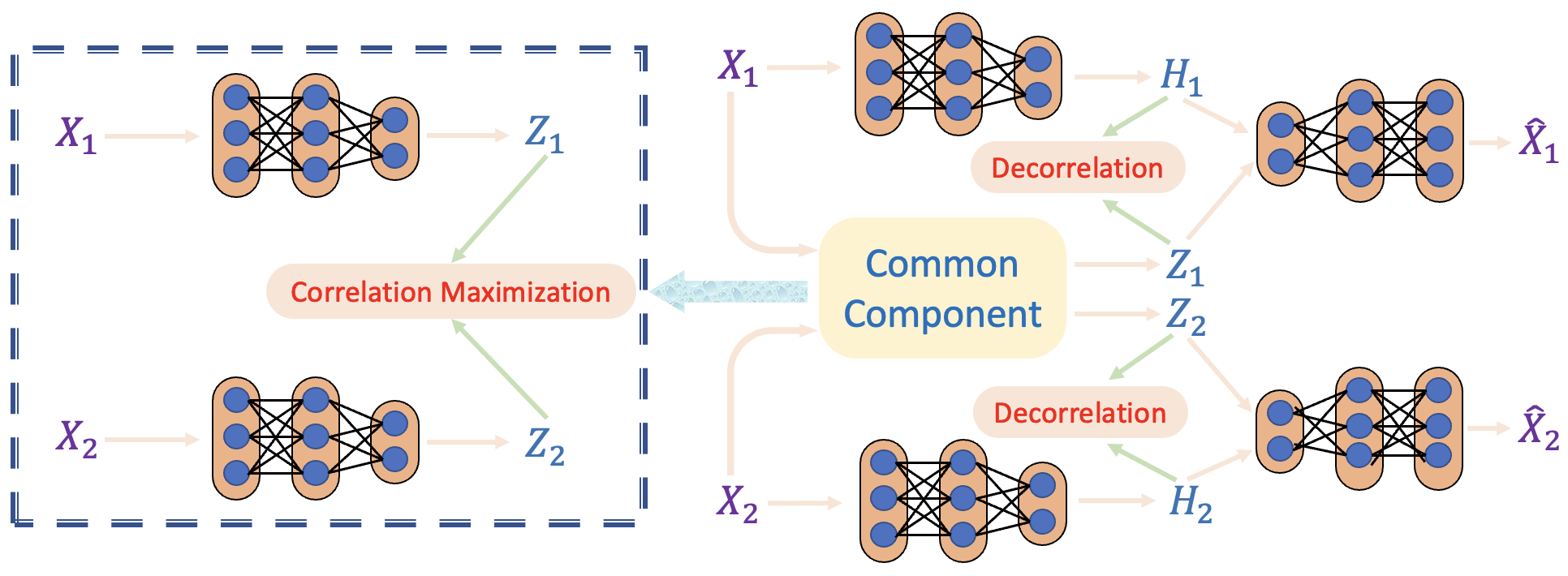}
\caption{An overview of \MUCMM. The inputs, $X_1$ and $X_2$, are extracted features from pre-trained networks. The training is done in two stages: 1) we first learn the common component (shown on the left) by optimizing the loss in Eq~\ref{loss_common} and output the common latent representations $Z_1$ and $Z_2$; and 2) we then freeze the common component and use it for learning the individual component (shown on the right) by optimizing the loss in Eq.~\ref{loss_indi}, and output the individual latent representations $H_1$ and $H_2$. Last but not least, we reconstruct $X_1$ and $X_2$ using the (concatenation of) respective common and individual representations.}
\label{model_diagram}
\vspace{-4mm}
\end{figure*}
\vspace{-0.5mm}
\section{Approach}
\label{model}
\paragraph{Notation} We use upper case letters to denote matrices and higher-order tensors (e.g. batch of matrices or images). For a subscripted matrix $X_i$, we use $x_{i,j}$ (column form) to denote the $j$-th \emph{row} of $X_i$, and $x^{i,j}$ (column form) to denote the $j$-th \emph{column} of $X_i$. For a generic modality (e.g. text) $M_i$, $m_{i,j}$ denotes the $j$-th item of $M_i$. We use bold numbers to denote vectors, matrices, or tensors of constant elements (e.g. $\textbf{0}_{p\times q}$ denotes a $p\times q$ matrix of zeros, \textbf{I} denotes the identity matrix). 

 For concreteness, we assume that we have two modalities of $n$ pairs of data samples. The first modality, $M_1$, consists of images and the second, $M_2$, consists of text. We also assume we have two pre-trained networks, $\psi_1$ and $\psi_2$, for extracting feature representations of $M_1$ and $M_2$, respectively. Although possible, we do \emph{not} fine-tune the pre-trained networks in this work. Therefore, for our purpose, the inputs to our model are the extracted feature representations, $X_1 = \psi_1(M_1)$ and $X_2 = \psi_2(M_2)$ where $X_1 \in R^{n\times d_1}$ and $X_2 \in R^{n\times d_2}$ with $d_1$ and $d_2$ denoting the dimensions of the extracted features. We next introduce \MUCMM\ in terms of its two components: the common component for capturing the common structure, and the individual component for capturing the individual structure.
 
\subsection{Common Component}
\label{common_struc}
\paragraph{Common latent representations} We introduce a reformulated deep CCA to learn the latent representations for the common structure between two modalities. Deep CCA learns a separate function, $f_{\theta_i}: X_i \rightarrow Z_i$ for $i\in \{1,2\}$ where $Z_i\in R^{n\times k}$ and $k$ denotes the dimension of the common latent representations, to project each modality onto a latent space where the modalities are maximally correlated. We denote $\hat{Z}_i$ as the column-whitened (i.e. each column of $\hat{Z}_i$ has mean of 0 and variance of 1) $Z_i$ for $i\in \{1,2\}$. There are different, but equivalent, formulations of the deep CCA objective but the most relevant one is the following~\cite{Chang2017ScalableAE,Wang2016OnDM}
\begin{equation}\label{dcca}
\begin{aligned}
\max_{\theta_1,\theta_2} \quad & \frac{1}{n}\text{Tr}\left(\hat{Z}_1^T\hat{Z}_2\right)\\
\textrm{s.t.} \quad &  \frac{1}{n}\hat{Z}_1^T\hat{Z}_1 = \frac{1}{n}\hat{Z}_2^T\hat{Z}_2 = \textbf{I}\\
\end{aligned}
\end{equation}
where the mapping functions $f_{\theta_1}$ and $f_{\theta_2}$ are parameterized by neural networks. In words, the objective in Eq.~\ref{dcca} seeks pairs of common latent representations of the original modalities that are maximally correlated whereas the orthogonality constraints ensure the representational efficiency of the learned common latent space and avoid trivial solutions. 

Denote $\mathcal{C} = \hat{Z}_1^T\hat{Z}_2/n$ as the sample correlation matrix. We propose to combine the objective from Zbontar \etal \cite{Zbontar2021BarlowTS} with the deep CCA constraints and obtain the following alternative formulation
\begin{equation}\label{barlow_dcca}
\begin{aligned}
\min_{\theta_1,\theta_2} \quad & \sum_i \left(1-\mathcal{C}_{ii}\right)^2\\
\textrm{s.t.} \quad &  \text{same constraints as in (1)}\\
\end{aligned}
\end{equation}
where $\mathcal{C}_{ii}$ denotes the $i$-th diagonal element of $\mathcal{C}$. We note that the objectives in Eq.~\ref{dcca} and Eq.~\ref{barlow_dcca} are mathematically equivalent, as the correlation between two whitened representations is bounded within $[-1,1]$; thus maximizing the sum of the correlations between pairs of whitened representations is equivalent to minimizing the sum of the differences between the correlations and their upper bound (which is 1). Despite the conceptual equivalence, the formulation in Eq.~\ref{barlow_dcca} helps the model better avoid local maxima during training by offering the model an explicit goal to reach (as opposed to simply maximizing the sum of the correlations without informing the model what the upper limit is). We demonstrate the advantage of this formulation in Sec.~\ref{barlow_advantage}. Denote $\hat{\Sigma}_1 = \hat{Z}_1^T\hat{Z}_1/n$ and $\hat{\Sigma}_2 = \hat{Z}_2^T\hat{Z}_2/n$, we form the Lagrangian of the optimization problem detailed in Eq.~\ref{barlow_dcca} and use it as the overall loss for learning the common component of \MUCMM
\begin{equation}
    \mathcal{L}_{\text{comm}} = \sum_i \left(1-\mathcal{C}_{ii}\right)^2 + \lambda_1 \sum_{i\neq j} \left(\hat{\Sigma}_1\right)_{i,j}^2  + \lambda_2 \sum_{i\neq j} \left(\hat{\Sigma}_2\right)_{i,j}^2
\label{loss_common}
\end{equation}
where $\left(\hat{\Sigma}_i\right)_{i,j}$ denotes the $(i,j)$-th entry of $\hat{\Sigma}_i$ for $i\in \{1,2\}$.

\paragraph{Gradient maps} For an input sample pair $\{m_{1,i},m_{2,i}\}$, we described how we learn the common latent representations pair $\{z_{1,i},z_{2,i}\}$ that represents the shared common structure. Nevertheless, those latent representations are abstract and difficult to appreciate. To understand what they capture more intuitively, it is necessary to visualize the learned common structure in the original data space. To this end, we introduce the following score to summarize the common structure
\begin{equation}\label{common_score}
s_i = \left(\Tilde{z}_{1,i}^T\Tilde{z}_{2,i}\right) \cdot \left({z}_{1,i}^T{z}_{2,i}\right)
\end{equation}
where $\Tilde{z}_{1,i}$ and $\Tilde{z}_{2,i}$ denote the $l_2$-normalized ${z}_{1,i}$ and ${z}_{2,i}$, respectively. We can then use $s_i$ to compute the gradient maps $\partial s_i/\partial m_{1,i}$ and $\partial s_i/\partial m_{2,i}$ (i.e. we backpropagate through both the common component and the pre-trained feature extractor). Intuitively, $s_i$ measures how similar the common representations of the $i$-th pair of input samples are in the common latent space. In turn, taking the first modality as an example, $\partial s_i/\partial m_{1,i}$ encapsulates which part of $m_{1,i}$ contributes the most to the learned common structure (with respect to the other view). We formalize this intuition in the case of linear $f_{\theta_i},\ i\in \{1,2\}$ in Sec.~\ref{joint_linear_case_proof}. We also note that in practice, one can use more sophisticated tools like GradCAM \cite{Adebayo2018SanityCF} in lieu of the raw gradients whenever appropriate, which we expound on in Sec.~\ref{experiments}. 

\subsection{Individual Component}
\label{indi_struc}
\paragraph{Individual latent representations} In order to capture the individual structures that are specific to the modalities, we learn $g_{\kappa_i}: X_i \rightarrow H_i$ for $i\in \{1,2\}$ where $H_i\in R^{n\times q}$ and $q$ denotes the dimension of the individual latent space that represents the individual structure for each view. We again denote $\hat{H}_i$ as the column-whitened $H_i$. To ensure that the learned individual latent space is uncorrelated with the common latent space, we enforce the sample correlations between the individual and common latent representations to be zeros for each view. However, this constraint alone would not guarantee a meaningful $H_i$ as there can be multiple, potentially infinitely many, $H_i$ that satisfy this constraint. To ensure that $H_i$ is faithful to the original data, we learn $H_i$ in such a way that using both the common and individual representations (i.e. $Z_i$ and $H_i$) would lead to a good reconstruction of the input $X_i$. We therefore arrive at the following optimization problem (for $i\in \{1,2\})$
\begin{equation}\label{indi_opti}
\begin{aligned}
\min_{\xi_i,\kappa_i} \quad & \sum_{j=1}^n ||d_{\xi_i}\left([{z}_{i,j};{h}_{i,j}]\right)-x_{i,j}||_2^2\\
\textrm{s.t.} \quad &  \frac{1}{n}\hat{Z}_i^T\hat{H}_i = \mathbf{0}_{k\times q}\\
\end{aligned}
\end{equation}
where $d_{\xi_i}:[{z}_{i,\cdot};{h}_{i,\cdot}] \rightarrow x_{i,\cdot}$ is the decoding function that reconstructs each sample in the $i$-th view (i.e. $x_{i,\cdot}$) using the concatenation of the common and individual representations for that sample. This way the model is incentivized to learn ${H}_i$ that contains useful information about $X_i$ while being uncorrelated with the already-learned $Z_i$. Denote $\hat{\Delta}_i =  \hat{Z}_i^T\hat{H}_i/n$, the overall loss for learning the individual structure for each $i\in \{1,2\}$ is 
\begin{equation}
    \mathcal{L}^i_{\text{indi}} = \sum_{j=1}^n ||d_{\xi_i}\left([{z}_{i,j};{h}_{i,j}]\right)-x_{i,j}||_2^2  + \nu_i  \sum_{k,l} \left(\hat{\Delta}_i\right)_{k,l}^2
\label{loss_indi}
\end{equation}
where $\left(\hat{\Delta}_i\right)_{k,l}$ denotes the $(k,l)$-th entry of $\hat{\Delta}_i$ for $i\in \{1,2\}$.

\paragraph{Gradient maps} For a pair of input samples $m_{1,j}$ and $m_{2,j}$ where $1\leq j\leq n$, it might seem straightforward to define the individual scores as the $l_2$ norm of the individual representations. However, such a score cannot generalize to unseen data at \emph{test} time, as the concept of individuality is ill-posed when not considering commonality. To this end, we further project the individual representations onto the subspace orthogonal to the common structure and use the $l_2$ norm of the projected representations as the individual scores
$$r_{i,j} = \frac{1}{2}||P^{\perp}h_{i,j}||_2^2\ ,\ \ \ \ P^{\perp} = \left(\mathbb{I}-\sum_{i=1}^2 z_{i,j}^*(z_{i,j}^*)^T\right)$$
where $\{z_{1,j}^*,z_{2,j}^*\}$ denotes the orthogonalized version of $\{z_{1,j},z_{2,j}\}$ (so that $P^{\perp}$ is an orthogonal projector). We will illustrate the intuition in the case of linear $g_{\kappa_i}$ in Sec.~\ref{indi_linear_case_proof}.

\subsection{Gradient Regularization}
\label{regu}
 We discuss the gradient regularization for the common component, as the same applies for the individual component. Take the $j$-th sample in $M_1$ (i.e. $m_{1,j}$) as an example, recall that once having defined the common score $s_j$, we can visualize the gradient map $\partial s_j/\partial m_{1,j}$. A simple chain rule reveals $\partial s_j/\partial m_{1,j} = \left(\partial s_j/\partial x_{1,j}\right)\cdot \left(\partial x_{1,j}/\partial m_{1,j}\right)$
where $x_{1,j}$ is the extracted feature representation of $m_{1,j}$ using the pre-trained network $\psi_1$. Since we do \emph{not} fine-tune $\psi_1$, we only have control over $\partial s_j/\partial x_{1,j}$. 

The extracted feature representations from pre-trained networks are usually arduous to interpret. Moreover, the extracted feature variables can be highly correlated/entangled with each other. To obtain a more accurate gradient map of $m_{1,j}$, we need to ensure $f_{\theta_1}$ focuses on the right (group of) feature variables in $x_{1,j}$, i.e., we want $\partial s/\partial x_{1,j}$ to only place large gradients on feature variables that indeed represent the common structure in $m_{1,j}$. Consequently, inspired by Rifai \etal~\cite{Rifai2011ContractiveAE}, we regularize the gradient of the score with respect to the extracted feature representations during training with the following updated objective of Eq.~\ref{barlow_dcca}
\begin{equation}\label{barlow_dcca_final}
\begin{aligned}
\min_{\theta_1,\theta_2} \quad & \sum_i \left(1-\mathcal{C}_{ii}\right)^2 + \gamma \cdot \sum_{i\in \{1,2\}}\sum_{j=1}^n \norm{\frac{\partial s_j}{\partial x_{i,j}}}_{\text{regu}}\\
\textrm{s.t.} \quad &  \text{same constraints as in (1)}\\
\end{aligned}
\end{equation}
where $\norm{\cdot}_{\text{regu}}$ denotes some form of regularization. We make two notes. First, regularizing the gradients incentivizes the model to be cost-efficent and find only the most relevant extracted feature variables that represent the common structure. Second, the choice of regularization matters. It is known that elastic net regularization \cite{Zou2005RegularizationAV} has the grouping effect (i.e., it is able to exclude groups of correlated variables), which we find most applicable due to the aforementioned black-box nature of those feature variables and the potentially high correlations among them. We provide further insights into the connection between the introduced gradient regularization and the traditional regression regularization techniques in the Appendix. 
\section{Gradient Maps in the Linear Setting}
\label{linear_case_proof}
Beyond the intuitions offered in Sec.~\ref{model}, we explore the theoretical properties of the gradient maps in a linear setting. We show that when $f_{\theta_i}$ and $g_{\kappa_i}$ are linear for $i\in \{1,2\}$, computing gradient maps is equivalent to projecting the data onto subspaces that encapsulate the desired structures. 
We assume we have two modalities, $X_1 \in R^{n\times d_1}$ and $X_2 \in R^{n\times d_2}$, where $n$ denotes the number of samples and $d_i$ denotes the feature dimensions of the $i$-{th} modality. 
\vspace{-3.5mm}
\paragraph{Visualization through linear projections} Another way of visualizing the common and individual structures of multimodal data is to find the two bases that span the subspaces for the respective two structures, linearly project the original modalities onto those subspaces, and visualize the projections \cite{Feng2017AnglebasedJA,Lock2011JOINTAI}. We restrict the following discussion to the common structure as the same applies for the individual structures. Deep CCA in essence attempts to find two sets of bases, $\mathcal{U}=\{u_1,\dots,u_k\}$ for $X_1$ where $u_i\in R^{d_1}$ and $\mathcal{V}=\{v_1,\dots,v_k\}$ for $X_2$ where $v_i\in R^{d_2}$, such that the two subspaces spanned by $\mathcal{U}$ and $\mathcal{V}$, respectively, are as close to each other as possible. Feng et.al \cite{Feng2017AnglebasedJA}, on the other hand, uses perturbation theory to find $\mathcal{U}$ and $\mathcal{V}$. Once both sets of bases are determined, one can compute the common structure, for example for each sample $x_{1,j}$ in $X_1$, by projecting the sample onto this subspace through $x_{1,j}^C = x_{1,j}^TP$ where $P = \sum_i u_iu_i^T$. Denoting $U\in R^{d_1\times k}$ and $V\in R^{d_2\times k}$ as matrices whose columns consist of $\{u_1,\dots,u_k\}$ and $\{v_1,\dots,v_k\}$, another way of viewing this projection is to first compute the coordinates of $x_{1,j}$ in the subspace spanned by the basis $U$ through $z_{1,j}^T = x_{1,j}^T U$, then each projected sample is the weighted (by the computed coordinates) sum of the basis, i.e. $x_{1,j}^C = z_{1,j}^T U^T$. Nevertheless, finding the common structure for each modality only using its basis alone does not generalize to unseen (test) data as it is not possible for the model to determine what the modalities have in \emph{common} without even considering the other modality. We show that, in the linear case, computing the gradient maps using the score we introduced are akin to projecting the modalities while appropriately incorporating both sets of bases at test time. 

\subsection{Gradient for the Common Structure}
\label{joint_linear_case_proof} 
We parameterize $f_{\theta_1}$ and $f_{\theta_2}$ using single-layer linear neural networks whose weights are $U$ and $V$, respectively (denoted as $f_{U}$ and $f_V$). We thus have $Z_1 = f_U(X_1) = X_1U$ and $Z_2 = f_V(X_2) = X_2V$. We learn $f_{U}$ and $f_V$ in terms of the objective detailed in Eq.~\ref{barlow_dcca}. Since each pair of samples $\{x_{1,j},x_{2,j}\}_{j}$ in the (batch of) data is independent of each other, without loss of generality, we analyze the gradient maps for the $j$-th pair of samples. 

For the $j$-th pair of samples, we write $z_{1,j}^T = x_{1,j}^T U$ and $z_{2,j}^T = x_{2,j}^T V$. We then write $\Tilde{z}_{1,j}^T= z_{1,j}^T/||z_{1,j}^T||_2$ and  $\Tilde{z}_{2,j}^T = z_{2,j}^T/||z_{2,j}^T||_2$ as the corresponding $l_2$-normalized latent representations. As an example, we compute $\partial s_j/\partial x_{1,j}$ where $s_j$ is defined as in Eq.~\ref{common_score} (see Appendix for derivations)
\begin{equation}
\label{common_grad}
\begin{aligned}
\frac{\partial s_j}{\partial x_{1,j}^T} =  \left(\frac{\partial s_j}{\partial x_{1,j}}\right)^T
= &  \frac{{z}_{1,j}^T{z}_{2,j}}{||z_{1,j}||_2} \left[\Tilde{z}_{2,j}^TU^T -\left(\Tilde{z}_{1,j}^T\Tilde{z}_{2,j}\right)\Tilde{z}_{1,j}^T U^T\right]\\
&\quad \quad \quad + \left(\Tilde{z}_{1,j}^T\Tilde{z}_{2,j}\right) z_{2,j}^TU^T
\end{aligned}
\end{equation}
To facilitate understanding, we interpret Eq.~\ref{common_grad} in two extremes. If $x_{1,j}$ and $x_{2,j}$ do not have much in common, i.e. the similarity between their corresponding normalized latent representations is close to 0 $\left(\Tilde{z}_{1,j}^T\Tilde{z}_{2,j} \rightarrow 0\right)$, both $\partial s_j/\partial x_{1,j}^T$ and $\partial s_j/\partial x_{2,j}^T$ consequently tend to 0, which is what we would have expected as none of the extracted feature variables contributes to the common score. On the other hand, if $x_{1,j}$ and $x_{2,j}$ are very similar to each other, resulting in the similarity between their corresponding normalized latent representations being close to 1 $\left(\Tilde{z}_{1,j}^T\Tilde{z}_{2,j} \rightarrow 1\right)$, we then have (see Appendix for derivations)
$$\frac{\partial s_j}{\partial x_{1,j}^T} \propto x_{1,j}^T UU^T,\quad \frac{\partial s_j}{\partial x_{2,j}^T} \propto x_{2,j}^T VV^T $$
which is in line with the perspective of finding the common structure through linear projections whereas the weights of the linear neural networks are the respective sets of bases. Any in-between scenario can be viewed as finding the common structure in each modality by incorporating its own basis and the basis from the other modality. This desired behavior of the gradient map in the linear setting motivates our definition of the score function defined in Eq.~\ref{common_score} for the general nonlinear case.

%
%

\subsection{Gradient for the Individual Structures}
\label{indi_linear_case_proof}
We demonstrate our approach for finding the individual structures (that are specific to each input modality) in the case of linear neural nets with one layer. In this case, we seek another two sets of bases, $\mathcal{W}=\{w_1,\dots,w_k\}$ for $X_1$ where $w_i\in R^{d_1}$ and $\mathcal{S}=\{s_1,\dots,s_k\}$ for $X_2$ where $s_i\in R^{d_2}$, that span the subspaces that encompass the individual structures of $X_1$ and $X_2$, respectively. We denote $W\in R^{d_1\times k}$ and $S\in R^{d_2\times k}$ as matrices whose columns consist of $\{w_1,\dots,w_k\}$ and $\{s_1,\dots,s_k\}$. Same as before, we view $W$ and $S$ as the weights of the linear neural networks that parameterize the functions $g_{\kappa_1}$ and $g_{\kappa_2}$, denoted as $g_{W}$ and $g_{S}$, respectively. We focus the discussion below on $X_1$ as the same applies to $X_2$. 
\vspace{-4.3mm}
\paragraph{The Orthogonality Constraint in Eq.~\ref{indi_opti}} Let $H_1 = g_W(X_1)= X_1W$. Recall that in Eq.~\ref{indi_opti}, the reconstruction objective encourages the model to learn meaningful $H_1$ whereas the constraint is to ensure $H_1$ captures different information as opposed to $Z_1$. We make two notes. First, we note that $\hat{H}_1$ and $\hat{Z}_1$ in Eq.~\ref{indi_opti} are assumed to be column-whitened. However, as whitening does not change the direction of the vectors, we can rephrase the constraint in Eq.~\ref{indi_opti} as learning a $W$ such that $H_1$ is orthogonal to $Z_1$. Second, we note that the constraint in Eq.~\ref{indi_opti} requires each column of $H_1$ to be orthogonal to every column in $Z_1$ (and vice versa). We first delineate the ramification of enforcing the orthogonality constraint between the $i$-th column of $H_1$ (i.e. $h^{1,i}$) and the $j$-th column of $Z_1$ (i.e. $z^{1,j}$). We then discuss the orthogonality constraint as a whole. Writing the $i$-th column of $H_1$ and the $j$-th column of $Z_1$ as $h^{1,i} = X_1 w^{1,i}$ and $z^{1,j} = X_1 u^{1,j}$, we have the following lemma (see the Appendix for proof) 
\begin{lemma}
Define the Mahalanobis norm of a vector $x$ with respect to a positive definite matrix $A$ as $||x||_A^2 = x^T A x$. Learning a $w^{1,i}$ that results in $\left(h^{1,i}\right)^Tz^{1,j}/n = 0$ is equivalent to searching for a direction $w^{1,i}$ that is orthogonal to $u^{1,j}$ in a Mahalanobis sense, i.e. $||w^{1,i}+u^{1,j}||_{\Sigma}^2 = ||w^{1,i}||_{\Sigma}^2+||u^{1,j}||_{\Sigma}^2$ where $\Sigma = X_1^T X_1/n$. 
\end{lemma}
\noindent Therefore, if we further assume that $X_1$ is pre-whitened, we can interpret the orthogonality constraint in its entirety as learning a set of basis, $\{w_1,\dots,w_k\}$, for the individual structure such that each basis $w_i$ is orthogonal, measured in terms of the sample correlation matrix of $X_1$, to the entire set of basis $\left(\{u_1,\dots,u_k\}\right)$ that represents the common structure. 

\paragraph{Gradient Map for the Individual Structure} For the $j$-th sample in $X_1$ $\left(\text{i.e.}\ x_{1,j}\right)$, we compute the gradient map $\partial r_{1,j}/\partial x_{1,j}$ as (see Appendix for details)
$$\frac{\partial r_{1,j}}{\partial x^T_{1,j}} = \left(\frac{\partial r_{1,j}}{\partial x_{1,j}}\right)^T =  x_{1,j}^T W^{\perp}(W^{\perp})^T$$
where $W^{\perp} = W(P^{\perp})^T$.
In words, in the linear case, computing the gradient map is equivalent to first (orthogonally) projecting each coordinate of the original basis $W$ using $P^{\perp}$, and then projecting the sample onto the subspace that spanned by the projected set of basis $W^{\perp}$.

\section{Experiments}
\label{experiments}
We construct \MUCMM\ in PyTorch \cite{Paszke2019PyTorchAI} and demonstrate its efficacy through a range of experiments. When the inputs are images, we compute the saliency maps \cite{Adebayo2018SanityCF, Simonyan2013DeepIC,Alqaraawi2020EvaluatingSM, Tomsett2019SanityCF} in lieu of the gradient maps from the last convolutional layer of the pre-trained feature extractor for better visualization. We refer readers to the Appendix for experimental details like the data split, the network architectures, etc. 

\subsection{Alternative Formation of Deep CCA}
\label{barlow_advantage}
As detailed in Sec.~\ref{common_struc}, we propose an equivalent formulation of the deep CCA objective that is easier to optimize. We show that our learned common representations achieve higher cross-modality correlation. We also include extra results on the cross-modality recognition task, which we detail in the Appendix. 

 \begin{table}[h]
\begin{center}
\begin{tabular}[!t]{c| c c c c c }
 \hline
 $k$ & 50 & 100 & 200 & 500 & 1000\\
 \hline
 Upper Bound & 50 & 100 & 200 & 500 & 1000\\
 \hline\hline
 CCA \cite{Hardoon2004CanonicalCA} & 28.3 & 34.2 & 48.7 & 74.0 & - \\ 
 
 Deep CCA \cite{Andrew2013DeepCC} & 29.5 & 44.9 & 59.0 & 84.7 & - \\

 DCCAE \cite{Wang2016OnDM} & 29.3 & 44.2 & 58.1 & 84.4 & -\\

  CorrNet \cite{Chandar2015CorrelationalNN} & 44.5 & - & - & - & -\\

 SDCCA \cite{Wang2015StochasticOF} & \textbf{46.4} & 89.5 & 166.1 & 307.4 & -\\

 Soft CCA \cite{Chang2017ScalableAE} & 45.5 & 87.0 & 166.3 & 356.8 & 437.7\\ 

  \MUCMM\  & 46.1 & \textbf{89.8} & \textbf{169.8} & \textbf{422.3} & \textbf{638.2}\\ 
\hline
\end{tabular}
\caption{Sum/total cross-modality correlation.}
\label{mnist_corr}
\end{center}
\vspace{-8mm}
\end{table}

Following the experimental setups of CorrNet \cite{Chandar2015CorrelationalNN} and Soft CCA \cite{Chang2017ScalableAE}, we use the MNIST dataset. MNIST \cite{LeCun2005TheMD} consists of a training set of 60000, and a test set of 10000, images of 28$\times$28 hand-written digits. We treat the left and right halves of the images as the two modalities, and train the common component of \MUCMM\ on the training set and report the (mean) cross-modality (total/sum) correlation on the test set. We vary the dimension of the common latent space (i.e. $k$ in Sec.~\ref{common_struc}) and report the results based on the different choices of $k$. We note that we use the same network architecture as in CorrNet and Soft CCA for fair comparison.
The results for the cross-modality correlation task are shown in Tab.~\ref{mnist_corr}. We see that \MUCMM\ outperforms all other methods for nearly all choices of $k$. As aforementioned, the proposed objective in Eq.~\ref{barlow_dcca} is easier to optimize than the CCA objective (which is used in all methods in Tab.~\ref{mnist_corr} other than the Soft CCA) because it offers the model an explicit goal to reach (as opposed to maximizing a quantity without informing the model what the limit is). Soft CCA reformulated the CCA objective to an equivalent one of minimizing the ${L}_2$ distance between the common latent representations. However, since the $L_2$ distance between latent representations is of significantly different scale compared to statistical (de)correlation, our formulation does not require careful weighting among quantities of different scales. This advantage is especially evidential for the cross-modality recognition task of which we expound on in the appendix.
\begin{figure}
\begin{subfigure}[t]{0.23\textwidth}
    \includegraphics[width=\textwidth]{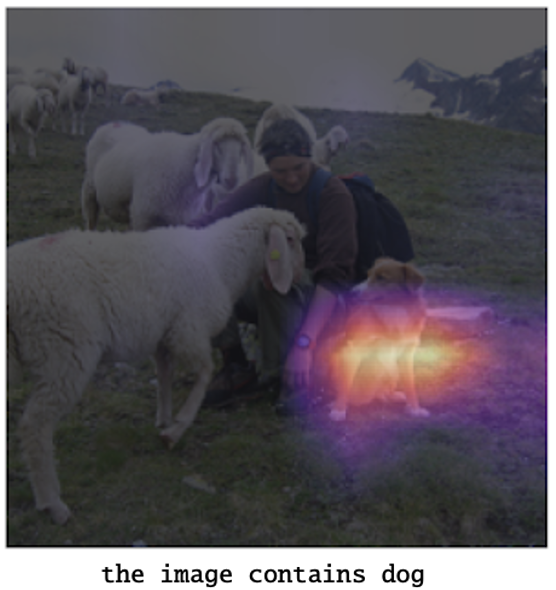}
    \includegraphics[width=\linewidth]{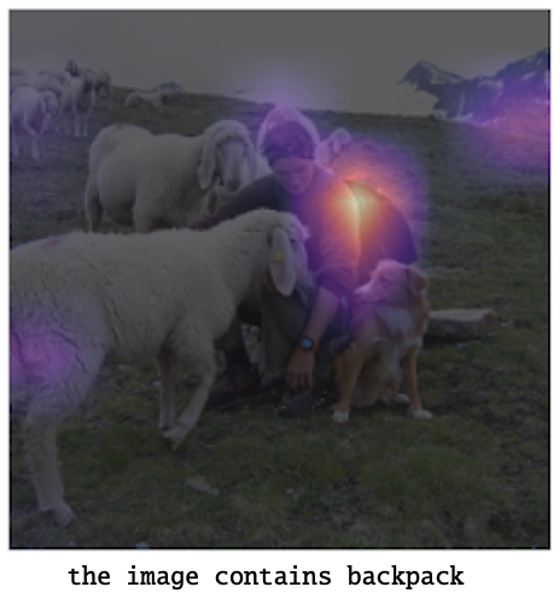}
\end{subfigure}
\smallskip 
\begin{subfigure}[t]{0.23\textwidth}
    \includegraphics[width=\linewidth]{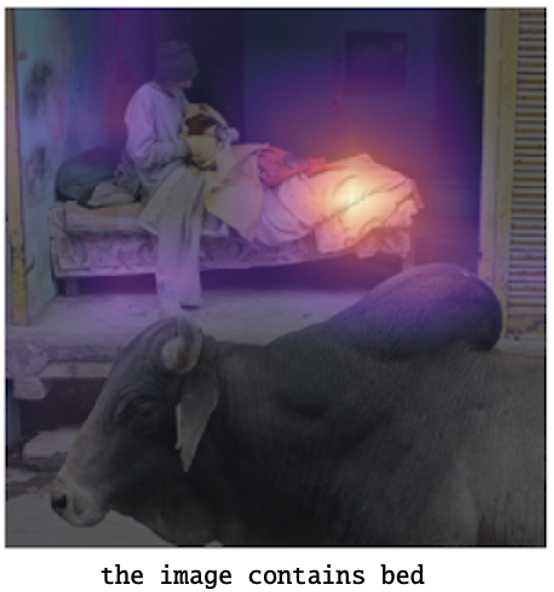}

    \includegraphics[width=\linewidth]{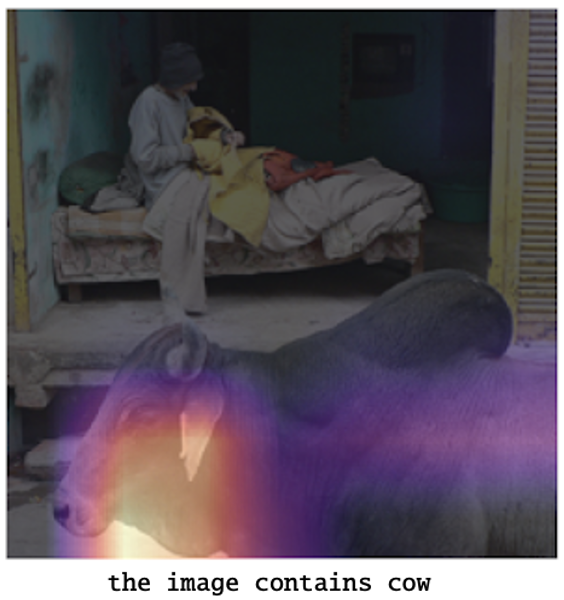}
\end{subfigure}
\smallskip 
\begin{subfigure}[t]{0.23\textwidth}
    \includegraphics[width=\linewidth]{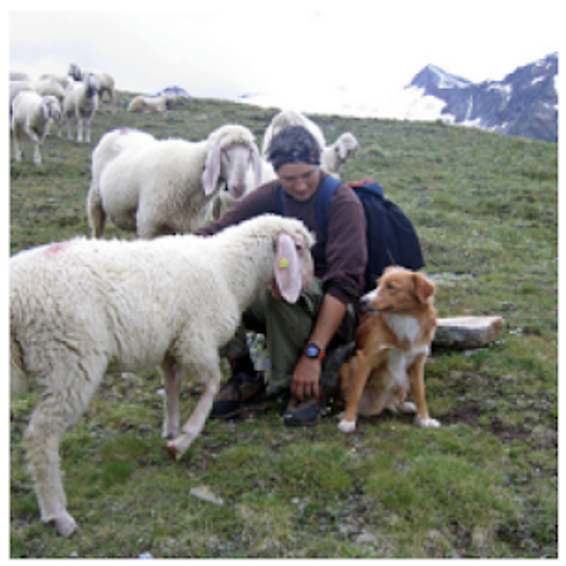}
    \caption{Example 1}
\end{subfigure}\hspace{\fill} 
\begin{subfigure}[t]{0.23\textwidth}
    \includegraphics[width=\linewidth]{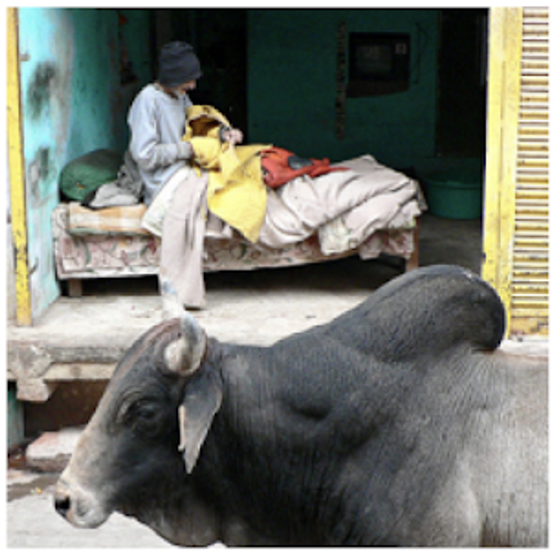}
    \caption{Example 2}
\end{subfigure}
\vspace{-1mm}
\caption{Two examples (columns) of OOI visualizations at test time are given. The first two rows contain the saliency maps for the exampled images given different text prompts, where the texts below the images are the prompts. The third row contains the original input images.}
\vspace{-1mm}
\label{coco_roi}
\end{figure}

\subsection{Paired COCO}
We design two experiments to demonstrate \MUCMM\ appropriately captures the common and individual structures between the modalities on the Common Objects in Context (COCO) dataset \cite{Lin2014MicrosoftCC}. COCO has 80 fine-grained classes with more than 200k labeled images. Each image in COCO is associated with a list of labels that indicates the objects in that image (as most images contain multiple objects). We refer readers to the Appendix for how we generate the saliency maps with \MUCMM.
\begin{figure*}
\setlength{\tabcolsep}{0.002\textwidth}
\begin{tabular}{@{}cccccc@{}}
\begin{subfigure}[t]{0.16\textwidth}
\includegraphics[width=\textwidth,valign=T]{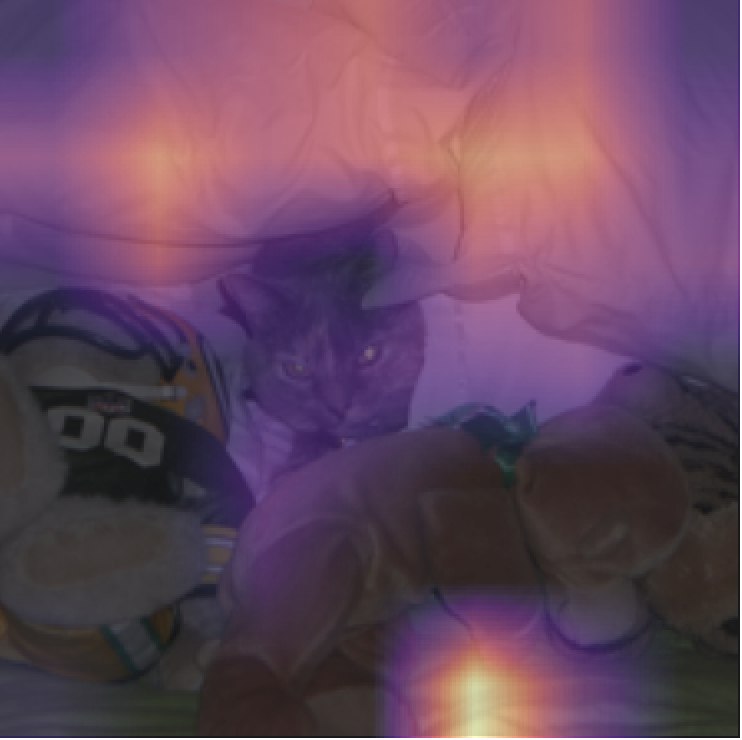}
\end{subfigure} &
\begin{subfigure}[t]{0.16\textwidth}
\includegraphics[width=\textwidth,valign=T]{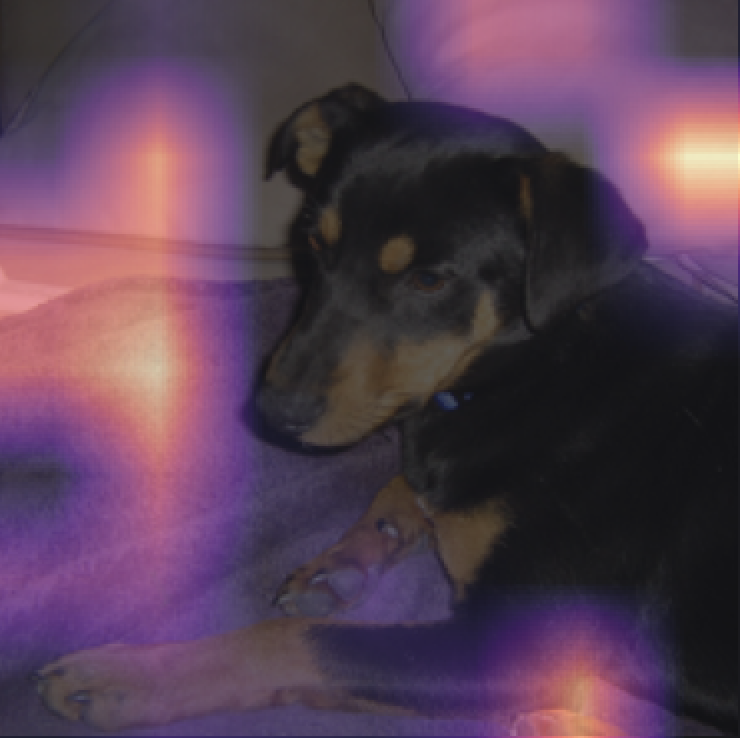}
\end{subfigure} &
\begin{subfigure}[t]{0.16\textwidth}
\includegraphics[width=\textwidth,valign=T]{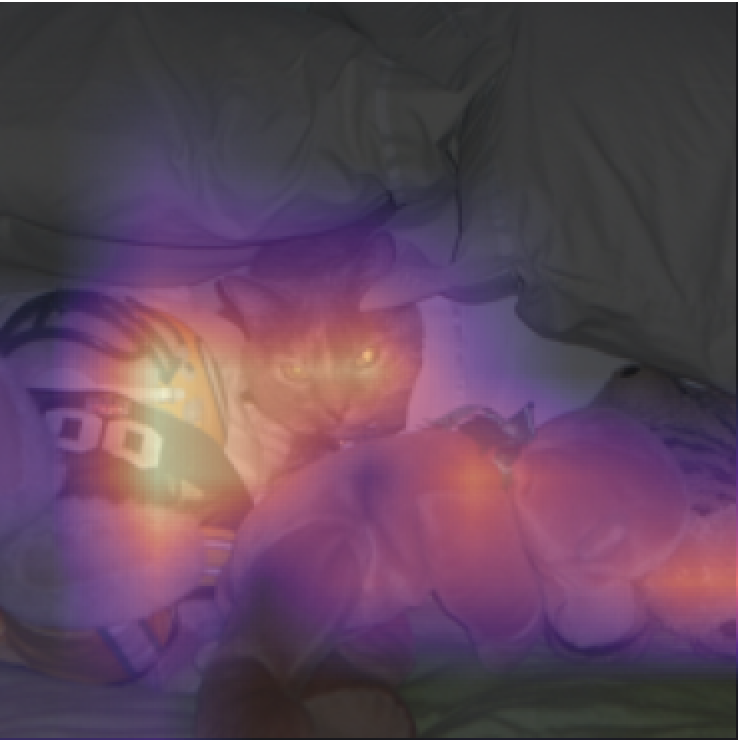}
\end{subfigure}&
\begin{subfigure}[t]{0.16\textwidth}
\includegraphics[width=\textwidth,valign=T]{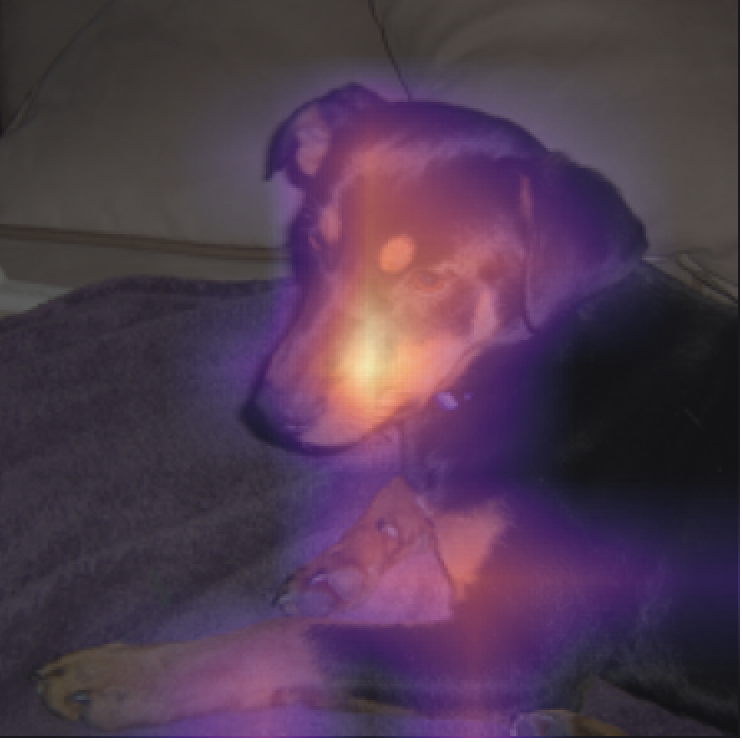}
\end{subfigure}&
\begin{subfigure}[t]{0.16\textwidth}
\includegraphics[width=\textwidth,valign=T]{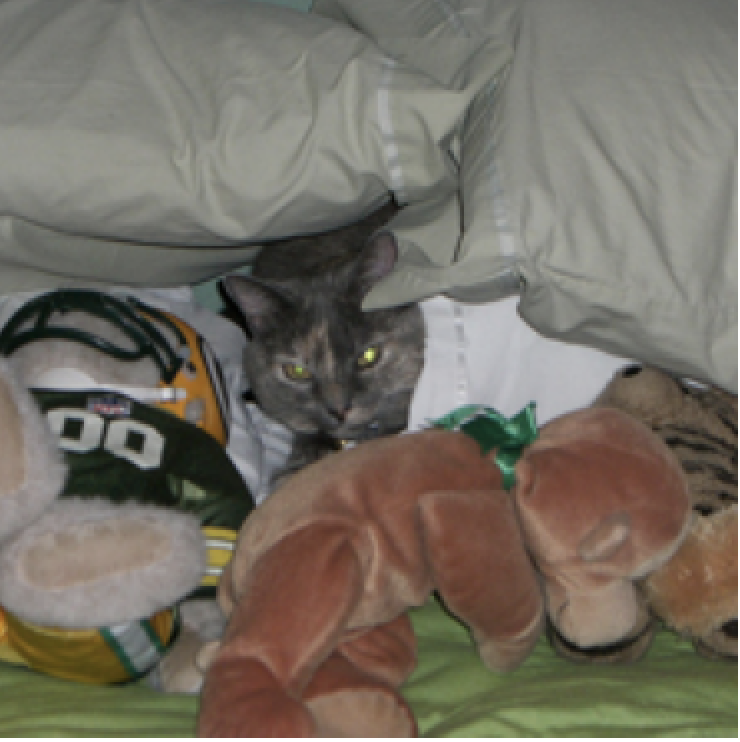}
\end{subfigure}&
\begin{subfigure}[t]{0.16\textwidth}
\includegraphics[width=\textwidth,valign=T]{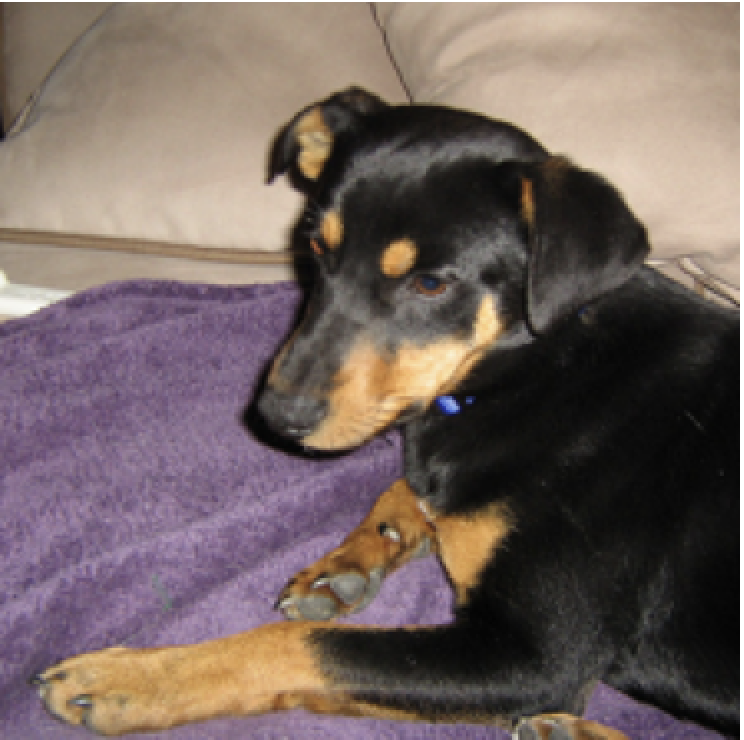}
\end{subfigure}\\
\begin{subfigure}[t]{0.16\textwidth}
\includegraphics[width=\textwidth,valign=T]{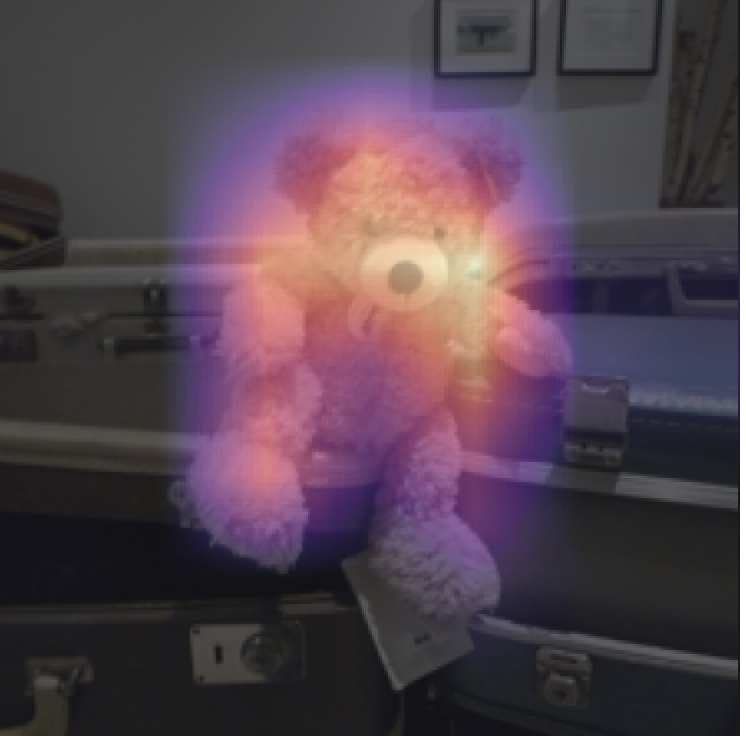}
\end{subfigure} &
\begin{subfigure}[t]{0.16\textwidth}
\includegraphics[width=\textwidth,valign=T]{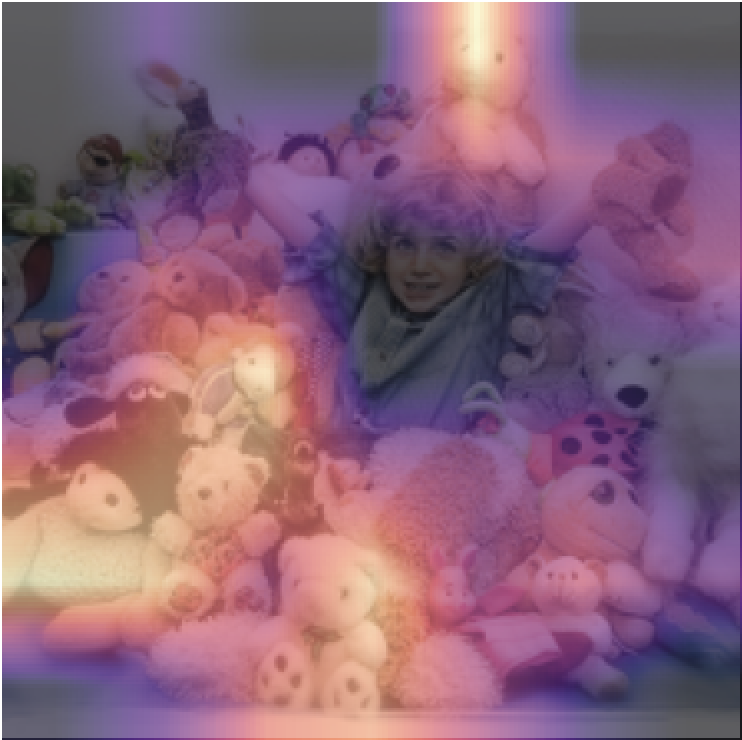}
\end{subfigure} &
\begin{subfigure}[t]{0.16\textwidth}
\includegraphics[width=\textwidth,valign=T]{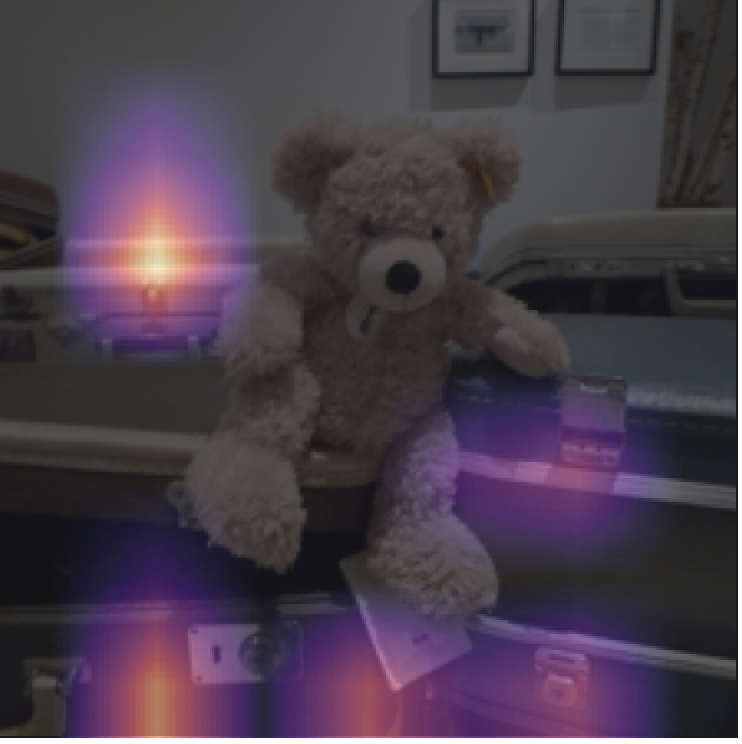}
\end{subfigure}&
\begin{subfigure}[t]{0.16\textwidth}
\includegraphics[width=\textwidth,valign=T]{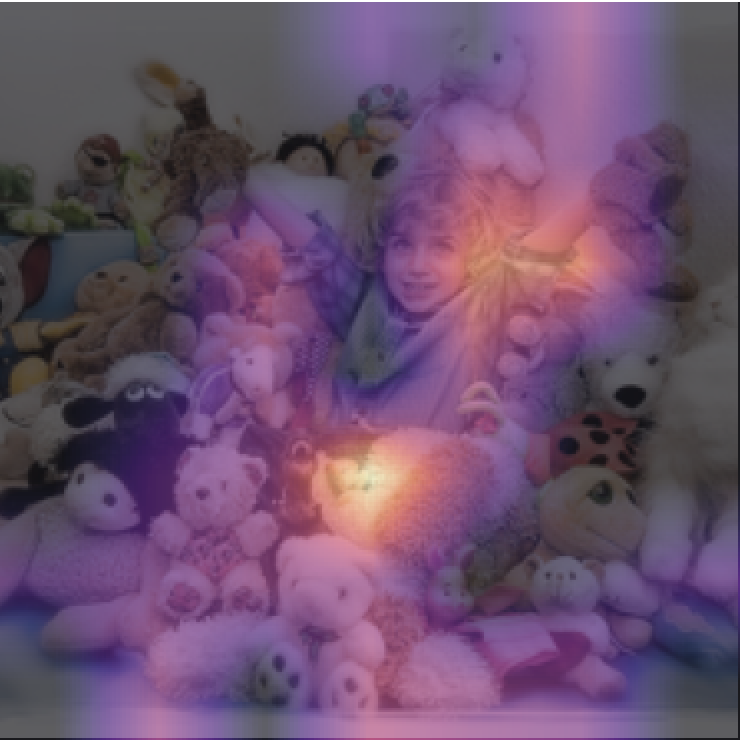}
\end{subfigure}&
\begin{subfigure}[t]{0.16\textwidth}
\includegraphics[width=\textwidth,valign=T]{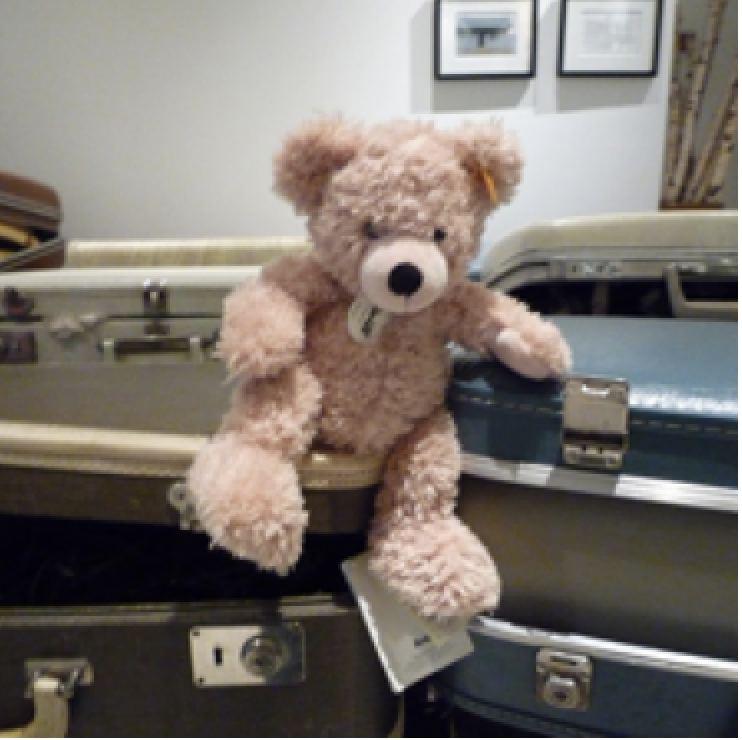}
\end{subfigure}&
\begin{subfigure}[t]{0.16\textwidth}
\includegraphics[width=\textwidth,valign=T]{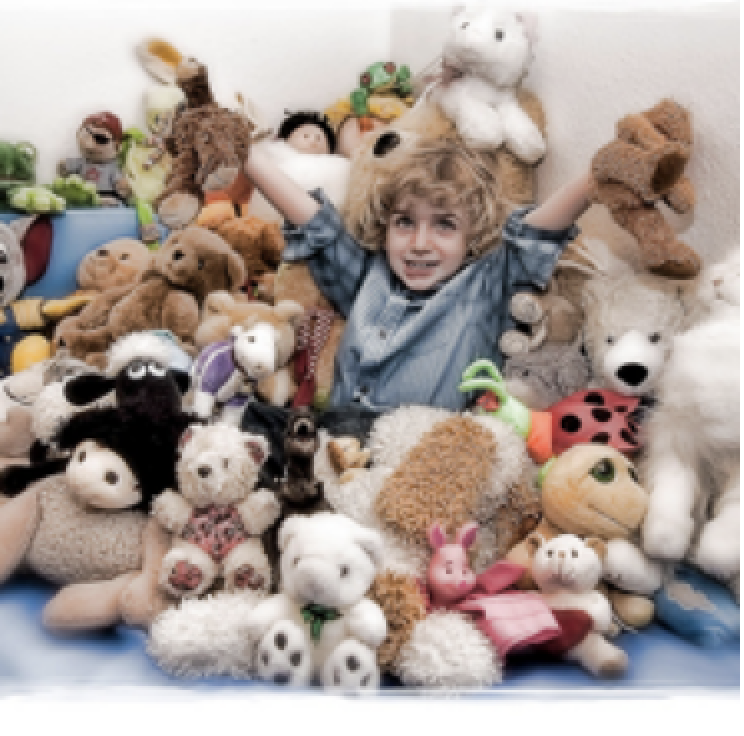}
\end{subfigure}\\
\multicolumn{2}{p{0.3\textwidth}@{}}{\subcaption{Common Structure}}&
\multicolumn{2}{p{0.3\textwidth}@{}}{\subcaption{Individual Structure}}&
\multicolumn{2}{p{0.3\textwidth}@{}}{\subcaption{Original Image Pair}}\\
\end{tabular}
\vspace{-3mm}
\caption{The results for two image pairs are shown (with each row representing a pair). The Grad-CAM visualizations for the common and the individual structures are shown in the first and the middle two columns, respectively. The original image pairs are shown in the last two columns. As one can see, the common objects (the pillows and blankets in the first, and the toys that face towards the camera in the second) are adequately captured by the common component. Once the common objects are identified, the individual component can then focus on the objects that are unique to the images.}
\label{paired_coco_ims}
\vspace{-3mm}
\end{figure*}
\vspace{-3mm}
\subsubsection{Object of Interest (OOI) Visualization}
\label{roi} 
We show that by providing \MUCMM\  with a text description of the object of interest (among potentially multiple objects) in an image, \MUCMM\ is able to capture the described object in that image. The modalities we use are image and text. We use the images for the image modality. For the text modality, we randomly sample a label from the list of labels for each image and construct the sentence ``\verb|the image contains {label}|'' as the input text. We use the pre-trained ResNet50 \cite{He2015DeepRL} and pre-trained BERT \cite{Devlin2019BERTPO} to extract the feature representations of the images and texts, respectively. 
The saliency maps for the input images, obtained from the common component of \MUCMM\ using the common score introduced in Sec.~\ref{common_struc}, are shown in Fig.~\ref{coco_roi}. As one can see, \MUCMM\ is able to locate the object described by the paired text in the image. Furthermore, \MUCMM\ can adapt to different text prompts at \emph{test} time and alter the object of interest highlighted in the same image correspondingly. 
\begin{table}[h]
\vspace{-0.5mm}
\begin{center}
    \begin{tabular}{|c|c|c|c|c|}
    \hline
    &\uni& CLIP&\MUCMM\ orig& \MUCMM\ w/ regu \\ \hline
    AP & 10.9& 20.6 & 18.1& \textbf{21.2} \\ \hline

    \end{tabular}
    \caption{Averge precision (AP) of the saliency maps on the COCO validation set. Unif. refers to the simple baseline of placing uniform weights over all pixels in each image, i.e. the ground truth proportion of the target object in the image. CLIP \cite{Radford2021LearningTV} refers to computing the common score (Eq.~\ref{common_score}) directly using the extracted features from a pre-trained CLIP model, and use that score to compute the saliency maps.} 
    \label{match_precision}
\end{center}
\vspace{-6mm}
 \end{table}
 
We also measure the quantitative quality of the saliency maps by calculating the average precision (AP) of them against the ground truth masks of the objects. For each saliency map, we normalize it so the gradient values sum to 1, and then calculate the proportion of the map that covers the ground truth object (hence precision) it is supposed to highlight. Since we compute a saliency map for each object in each image, AP is the precision averaged over all objects in all images in the test set. The result is shown in Tab.~\ref{match_precision}. We make three notes. First, \MUCMM\ performs substantially better than the uniform baseline. Second, regularized \MUCMM\ generalizes better on unseen data, corroborating that the proposed regularization improves \MUCMM's robustness by encouraging it to focus more on the relevant feature variables. Last but not least, regularized \MUCMM\ is on par with CLIP, a strong baseline for this task since CLIP is trained in a contrastive manner to distinguish true image-text pairs, which is only possible if CLIP learns to identify the common structure between the correct pairs. In comparison, \MUCMM\ can be incorporated with any pre-trained networks, as we have done here with ResNet50 and Bert, and achieves competitive performance. 

\vspace{-2mm}
\subsubsection{Pair of Images}
\begin{table*}[h]
\begin{center}
    \begin{tabular}{|c|c|c|c|c|c|}
    \hline
    &\uni& \baseresnet&\baseclip&\MUCMM\ orig& \MUCMM\ w/ regu \\ \hline
    $\text{AP}_C$ & 14.9 & 16.5 & 18.9 &  18.1& \textbf{22.7} \\ \hline
    $\text{AP}_{I}$ & 20.1 &  21.5 & 26.9 & 26.9& \textbf{27.5} \\ \hline
    $\text{AP}_{I}^{\text{obj}}$& 65.5  & 65.1 & 57.4 & 68.2& \textbf{70.1} \\ \hline

    \end{tabular}
    \caption{Average precision (AP) of the saliency maps on the test set, where $\text{AP}_C$ and $\text{AP}_I$ are that for the common and individual objects, respectively. $\text{AP}_{I}^{\text{obj}}$ is AP for the individual objects when \emph{not} considering the background, i.e. it indicates the average precision of the saliency map for the individual objects as opposed to all objects, disregarding the background. Unif. refers to the average ground truth proportions of the common/individual objects. \baseresnet and \baseclip\ refer to the baselines where we directly use the extracted features by applying the ResNet50 and CLIP-ResNet50 on both image modalities, respectively, for computing the saliency maps.} 
    \label{pair_precision}
\end{center}
\vspace{-6mm}
 \end{table*}

We illustrate that \MUCMM\ captures the common and individual structures between two image modalities. We construct a dataset of pairs of images using the COCO dataset. We generate each pair of images by first randomly sampling a class from the available 80 fine-grained classes. Given the sampled class, we randomly sample two images belonging to this class as the pair of images (note that the sampled images might belong to other classes as well because most images in COCO contain multiple objects). We choose the pre-trained original ResNet50 \cite{He2015DeepRL} and the CLIP ResNets50 \cite{Radford2021LearningTV} as the feature extractors for the paired images.

We visualize the learned structures for two sample pairs in Fig.~\ref{paired_coco_ims}. As one can see, \MUCMM\ adequately captures the common and individual objects within each image pair. As most machine learning models produce, though powerful, black-box latent representations of complex data that are not interpretable, being able to visually understand the latent representations learned by \MUCMM\ is highly desirable. 

Last but not least, we provide the AP of the saliency maps for the common and individual objects of the image pairs. For an image pair, we compute its precision by computing the pair of saliency maps for the common/individual objects, calculating the per-image precision within the pair, and averaging the per-image precision. The results are in Tab.~\ref{pair_precision}. 
We make four notes. First, we again see that \MUCMM\ performs better than the uniform baseline for all three APs. Second, for all three APs, \MUCMM\ with regularization outperforms directly using the pre-trained image features (\baseresnet and \baseclip\ in Tab.~\ref{pair_precision}) for computing the saliency maps. This is sensible, as the two baselines are computed directly from the extracted features and are not concept-aware of common versus individual. Third, the regularized \MUCMM\ again outperforms its vanilla counterpart, substantiating the importance of focusing on the right (group of) extracted feature variables. Last but not least, since background accounts for most parts of the image for most images and the individual component, while designed to capture aspects of the images that are not the common objects, can capture the backgrounds, we remove backgrounds in calculating the AP and report it as $\text{AP}_{I}^{\text{obj}}$ in Tab.~\ref{pair_precision}. We see that after removing backgrounds, \MUCMM, with or without regularization, focuses more on the individual objects as intended. All three baselines on the other hand performs worse than the \uni\ baseline as they do not have a sense of common versus individual and focus undesirably on the common objects when they are not supposed to. 
\vspace{-3mm}
\paragraph{Additional Study---Common Score} As mentioned in Sec.~\ref{common_struc}, we design the common score to be dependent on the common representations from both modalities as it is not possible to discern what is in common without examining both modalities for unseen data pairs. To show the importance of such design, we construct a test dataset where each image pair has no objects in common. We evaluate \MUCMM\ and the baselines using different choices of the common score on this dataset. For each image pair, we compute the pair of saliency maps based on the specific common score, normalize the saliency maps so that the maximum value of each saliency map is 1 (and all values are between 0 and 1), and compute the saliency magnitude by averaging the sum of each saliency map within the pair. Intuitively, the less the saliency maps focus on the images (as they are not supposed to when no common objects exist), the smaller the saliency magnitude will be. The results are in Tab~\ref{different_score_precision}. 
 \begin{table}[h]
\begin{center}
    \begin{tabular}{|c|c|c|c|c|}
    \hline
    & Eq.~\ref{common_score}  &$\Tilde{z}_{1,i}^T\Tilde{z}_{2,i}$& $||{z}_{\cdot,i}||_2^2$ \\ \hline
   \baseresnet & \textbf{16.2}  & 18.3&47.6 \\ \hline
    \baseclip & \textbf{2.6} & 4.0&11.7 \\ \hline
    \MUCMM\ orig & \textbf{15.7} & 16.6&41.4 \\ \hline
    \MUCMM\ w/ regu & \textbf{14.3} & 16.2&34.9 \\ \hline
    \end{tabular}
    \caption{The saliency magnitudes of different models with different common scores. Each row/column represents different combinations of model/common score, respectively. The smallest numbers (across each row) are bolded.} 
    \label{different_score_precision}
\end{center}
\vspace{-6mm}
 \end{table}
We make two notes. First, the saliency maps put much less weight on the images when the common score depends on the common representations from both modalities (the second and third columns), compared to only using the $l_2$ norm of the representations as the score (the last column). Second, the minimal saliency magnitude is achieved when using our proposed score (Eq.~\ref{common_score}) for all models (rows), corroborating the desired boundary behavior of Eq.~\ref{common_score} illustrated in Sec.~\ref{joint_linear_case_proof}. 
 
\section{Conclusion}
We introduce \MUCMM, a general framework for learning the common and individual structures of multimodal data that 1) is designed to be lightweight (i.e. it can be incorporated with any appropriate large pre-trained networks, requires relatively much less data to train, and achieves similar or even superior performances compared to specialized large networks trained on magnitudes of more data), 2) emphasizes interpretability, and 3) incorporates contractive regularization for better representation learning. The interpretability of \MUCMM\ comes in twofold: 1) To the best of our knowledge, this is the first attempt to visualize the learned structures in a model-agnostic manner (i.e. no constraints on the architectures of the pre-trained feature extractors or \MUCMM\ itself) by introducing novel scores that summarize the structures; and 2) Mathematical intuitions of the gradient maps are provided in the case of linear \MUCMM. Through a range of quantitative measures and visualizations, we demonstrated that \MUCMM\ learns the desired structures of multimodal data in a manner that is generalizable and adaptive at test time. 

\section{Appendix}
\subsection{Additional Insights}
\label{insights}
\subsubsection{Connecting Gradient Regularization to Regularized Regression}
\label{regu_grad_regression}
We now connect the introduced gradient regularization to regularized regressions in the linear setting. We again focus the following discussion on the gradient for the common structure as the same applies to that for the individual structures. 
We can rewrite Eq.~8 as
$$\frac{\partial S}{\partial x_{1,j}^T} = c^TU^T = \sum_{i=1}^k c_i u_i^T$$
where $c\in R^{n\times 1}$ encompasses the expressions that come before $U^T$ in all terms in Eq.~8 and $c_i$ is the $i$-th entry of $c$. In the special case of $k=1$, it is then easy to observe that
$$\norm{\frac{\partial S}{\partial x_{1,j}^T}}_{p} = |c| \norm{u_1}_p$$
where $\norm{\cdot}_{p}$ refers to the vector $p$-norm. In this case, regularizing the gradients is equivalent to regularizing $u_1$, which is akin to the regression parameters in that each element of $u_1$ reflects the importance of its corresponding extracted feature variable. Depending on the choice of $p$, such regularization includes lasso, ridge, and elastic net as special cases. In a traditional regression setting, regularizing the regression parameters would penalize the corresponding predictors; in our setting, regularizing $u_1$ enables the model to learn to ``ignore'' extracted feature variables that do not contribute enough to the common structure.

\subsubsection{Training the Components Separately}
We mentioned in the main manuscript that we first train the common component of \MUCMM\ to learn the common representations of the input modalities. We then freeze the common component and use it to train the individual component to learn the individual representations (recall Fig.~1 in the main manuscript). We explain our choice of training the two components in stages and point out some possibilities for future studies in this section. 

The main reason we train the two components in stages is because we found that when jointly training them, the model tends to only use either the common or the individual representations for reconstructing the inputs, but not both. This behavior is not surprising, because if the model can obtain a good reconstruction performance using only one type of representations, there is no incentive for the model to use both. Therefore, when we attempted to train both components together, we observed that the model \emph{only} uses the common representations for reconstructions, and thus learns individual representations that are orthogonal to the common representations but \emph{not} faithful to the original input modalities (as we mentioned in the main manuscript, there are multiple, potentially infinite many, individual representations orthogonal to the common representations; however, most of them would not be meaningful as they do not represent the original inputs). A potential quick fix to this issue is to reduce the representational capacity of the common and individual representations (e.g. by lowering their latent dimensions) to incentivise the model to use both representations, but in practice we find it difficult to tune this latent dimension hyperparameter, and different datasets/applications might require different latent dimensions. We therefore decide to train the two components separately in this work. 

For future studies, one potentially interesting direction is to regularize the representational capacities of those representations (using our introduced contractive regularization or some other differentiable regularizers). An important question in this case would be how to weight the regularizations for the common and individual representations in a data-driven manner so the model would not require manual tuning when applying it to new datasets. We plan on exploring this direction in future iterations of this work. 

\subsection{Derivations and Proof}
\label{proof}
\begin{prop} Defining the score for the common structure as
\begin{equation}\label{common_score_append}
s_i = \left(\Tilde{z}_{1,i}^T\Tilde{z}_{2,i}\right) \cdot \left({z}_{1,i}^T{z}_{2,i}\right)
\end{equation}
then in the case of linear encoding functions, $f_{\theta_1}$ and $f_{\theta_2}$, for the common structures, the gradient map of $s_i$ with respect to $x_{1,j}$ (as an example, as the same applies to $x_{2,j}$) is
\begin{equation}
\label{common_grad_append}
\begin{aligned}
\frac{\partial s_j}{\partial x_{1,j}^T} = & \left(\frac{\partial s_j}{\partial x_{1,j}}\right)^T\\
= &  \frac{{z}_{1,j}^T{z}_{2,j}}{||z_{1,j}||_2} \left[\Tilde{z}_{2,j}^TU^T -\left(\Tilde{z}_{1,j}^T\Tilde{z}_{2,j}\right)\Tilde{z}_{1,j}^T U^T\right]\\
&\quad \quad \quad + \left(\Tilde{z}_{1,j}^T\Tilde{z}_{2,j}\right) z_{2,j}^TU^T
\end{aligned}
\end{equation}
\end{prop}
\begin{proof}
First, we note that only ${z}_{1,j}$ and $\Tilde{z}_{1,j}$ depend on $x_{1,j}$. Applying the product rule results in 
\begin{equation}\label{beginning}
\begin{aligned}
\frac{\partial s_j}{\partial x_{1,j}} &= \left(\frac{\partial\left(\Tilde{z}^T_{1,j} \Tilde{z}_{2,j}\right)}{\partial x_{1,j}}\right) \cdot \left({z}_{1,j}^T{z}_{2,j}\right) +\\& \ \ \ \ \ \ \ \ \ \ \ \ \ \ \ \left(\Tilde{z}_{1,j}^T\Tilde{z}_{2,j}\right) \cdot \left(\frac{\partial\left({z}^T_{1,j}{z}_{2,j}\right)}{\partial x_{1,j}}\right)\\
\end{aligned}
\end{equation}
where the Jacobian matrices $\left(\partial z_{1,j}/\partial x_{1,j}\right), \left(\partial \Tilde{z}_{1,j}/\partial x_{1,j}\right) \in R^{k\times d_1}$. We note that since $z^T_{1,j} = x^T_{1,j} U$, we have
$$\frac{\partial\left({z}^T_{1,j}{z}_{2,j}\right)}{\partial x_{1,j}} = \frac{\partial \left(x^T_{1,j}U{z}_{2,j}\right)}{\partial x_{1,j}} = U{z}_{2,j}$$
Since $\Tilde{z}_{1,j} = z_{1,j}/||z_{1,j}||_2$, applying another product rule results in
\begin{equation*}
\begin{aligned}
\frac{\partial\left(\Tilde{z}^T_{1,j} \Tilde{z}_{2,j}\right)}{\partial x_{1,j}} &= \frac{\partial\left[\left({z}^T_{1,j} \Tilde{z}_{2,j}\right)/||z_{1,j}||_2\right]}{\partial x_{1,j}}\\
&= \frac{\partial\left({z}^T_{1,j} \Tilde{z}_{2,j}\right)/\partial x_{1,j}}{||z_{1,j}||_2} + \left({z}^T_{1,j} \Tilde{z}_{2,j}\right) \frac{\partial \left(1/||z_{1,j}||_2\right)}{\partial x_{1,j}}
\end{aligned}
\end{equation*}
Similar to before, for the first term above we have
$$\frac{\partial\left({z}^T_{1,j} \Tilde{z}_{2,j}\right)/\partial x_{1,j}}{||z_{1,j}||_2} = \frac{U \Tilde{z}_{2,j}}{||z_{1,j}||_2}$$

\noindent For the second term, applying the chain rule results in 
\begin{equation*}
\begin{aligned}
\frac{\partial \left(1/||z_{1,j}||_2\right)}{\partial x_{1,j}} &= - \frac{1}{||z_{1,j}||_2^2} \frac{\partial||z_{1,j}||_2}{\partial x_{1,j}}\\
&= - \frac{1}{2||z_{1,j}||_2^3}\frac{\partial||z_{1,j}||_2^2}{\partial x_{1,j}}\\
&= - \frac{1}{||z_{1,j}||_2^3}\cdot x_{1,j}^T UU^T
\end{aligned}
\end{equation*}
where we used the fact that $$||z_{1,j}||_2^2 = z_{1,j}^T z_{1,j} = x_{1,j}^TUU^T x_{1,j}$$
and 
$$ \frac{\partial\left(x_{1,j}^TUU^T x_{1,j}\right)}{\partial x_{1,j}} = 2x_{1,j}^T UU^T$$
We thus have
\begin{equation*}
\begin{aligned}
\frac{\partial\left(\Tilde{z}^T_{1,j} \Tilde{z}_{2,j}\right)}{\partial x_{1,j}} &= \frac{U \Tilde{z}_{2,j}}{||z_{1,j}||_2} - \frac{\left({z}^T_{1,j} \Tilde{z}_{2,j}\right)}{||z_{1,j}||_2^3}\cdot x_{1,j}^T UU^T\\
&= \frac{U \Tilde{z}_{2,j}}{||z_{1,j}||_2} -\frac{\left(\Tilde{z}^T_{1,j} \Tilde{z}_{2,j}\right)}{||z_{1,j}||_2}\cdot \Tilde{z}^T_{1,j}U^T
\end{aligned}
\end{equation*}
Plugging the results for $\left(\partial z_{1,j}/\partial x_{1,j}\right)$ and $\left(\partial \Tilde{z}_{1,j}/\partial x_{1,j}\right)$ back into Eq.~\ref{beginning} results in Eq.~\ref{common_grad_append}. 
\end{proof}

\begin{prop}
Define the score for the individual structure as 
$$r_{i,j} = \frac{1}{2}||P^{\perp}h_{i,j}||_2^2\ ,\ \ \ \ P^{\perp} = \left(\mathbb{I}-\sum_{i=1}^2 z_{i,j}^*(z_{i,j}^*)^T\right)$$
where $\{z_{1,j}^*,z_{2,j}^*\}$ denotes the orthogonalized version of $\{z_{1,j},z_{2,j}\}$ (so that $P^{\perp}$ is a proper orthogonal projector). In the case of linear encoding functions, $g_{\kappa_1}$ and $g_{\kappa_2}$, for the individual structures, the gradient map of $r_{1,j}$ with respect to $x_{1,j}$ (as an example, as the same applies to that of $r_{2,j}$ with respect to $x_{2,j}$) is
$$\frac{\partial r_{1,j}}{\partial x^T_{1,j}} = \left(\frac{\partial r_{1,j}}{\partial x_{1,j}}\right)^T =  x_{1,j}^T W^{\perp}(W^{\perp})^T$$
where $W^{\perp} = W(P^{\perp})^T$.
\end{prop}
\begin{proof}
We first write $r_{1,j}$ as 
$$r_{1,j} = \frac{1}{2}||P^{\perp}h_{1,j}||_2^2 = \frac{1}{2}h^T_{1,j}\left(P^{\perp}\right)^T P^{\perp}h_{1,j}$$
Since $h^T_{1,j} = x_{1,j}^TW$, we thus have 
$$\frac{\partial r_{1,j}}{\partial x^T_{1,j}} = \left(\frac{\partial r_{1,j}}{\partial x_{1,j}}\right)^T= x_{1,j}^T W^{\perp}(W^{\perp})^T$$
where $W^{\perp} = W(P^{\perp})^T$.
\end{proof}

\begin{lemma}
Define the squared Mahalanobis norm of a vector $x$ with respect to a positive definite matrix $A$ as $||x||^2_A = x^T A x$. Learning a $w^{1,i}$ that results in $(h^{1,i})^Tz^{1,j}/n = 0$ is equivalent to searching for a direction $w^{1,i}$ that is orthogonal to $u^{1,j}$ in a Mahalanobis sense, i.e. $||w^{1,i}+u^{1,j}||^2_{\Sigma} = ||w^{1,i}||^2_{\Sigma}+||u^{1,j}||^2_{\Sigma}$ where $\Sigma = X_1^T X_1/n$. 
\end{lemma}
\begin{proof}
We first rewrite
\begin{equation*}
\begin{split}
\frac{1}{n}(h^{1,i})^Tz^{1,j} &= \frac{1}{n}(w^{1,i})^TX_1^TX_1u^{1,j}\\
&= (w^{1,i})^T\Sigma u^{1,j}\\
\end{split}
\end{equation*}
Since $$||w^{1,i}+u^{1,j}||^2_{\Sigma} = |w^{1,i}||^2_{\Sigma} + ||u^{1,j}||^2_{\Sigma} + 2(w^{1,i})^T\Sigma u^{1,j}$$
enforcing $(w^{1,i})^T\Sigma u^{1,j} = 0$ is equivalent to finding a $w^{1,i}$ such that $||w^{1,i}+u^{1,j}||^2_{\Sigma} = |w^{1,i}||^2_{\Sigma} + ||u^{1,j}||^2_{\Sigma}$.
\end{proof}

\subsection{Additional Experiments}
\label{more_experiment}
\subsubsection{Alternative Formulation of Deep CCA---Cross-modality Recognition}
 \begin{table}[h]
\begin{center}
\begin{tabular}[!t]{c| c c c c c }
 \hline
 $k$ & 50 & 100 & 200 & 500 & 1000\\
 \hline
 Upper Bound & 50 & 100 & 200 & 500 & 1000\\
 \hline\hline
 CCA \cite{Hardoon2004CanonicalCA} & $\sim$ 70 & $\sim$ 70 & $\sim$ 70 & $\sim$ 70 & - \\ 
 
 Deep CCA \cite{Andrew2013DeepCC} & $\sim$ 72 & $\sim 75$ & $\sim 77$ & $\sim 79$ & - \\

 DCCAE \cite{Wang2016OnDM} & $\sim$ 71 & 75 & 76 & 79 & -\\

  CorrNet \cite{Chandar2015CorrelationalNN} & $\sim$ 80 & - & - & - & -\\

 SDCCA \cite{Wang2015StochasticOF} & $\sim$ 80 &$\sim$ 80 & $\sim$ 81 & $\sim$ 82 & -\\

 Soft CCA \cite{Chang2017ScalableAE} & $\sim$ 82 & $\sim$ 84 & $\sim$ 85 & $\sim$ 86 & $\sim$ 87\\ 

  \MUCMM\  & \textbf{87.1} & \textbf{88.3} & \textbf{88.8} & \textbf{90.1} & \textbf{90.9}\\ 
\hline
\end{tabular}
\caption{Cross-modality recognition accuracy (in \%). We obtain the results for other methods from Chang \etal \cite{Chang2017ScalableAE}, where the results are not exact (hence we use $\sim$ to indicate the results are close to the numbers listed) as Chang \etal only included a figurative illustration of the results instead of reporting the actual numerical numbers. The approximate results listed above are charitable in the sense that we deliberately tried to ``round up'' the result whenever it is not clear what the exact number is in their figure.}
\label{mnist_reg}
\end{center}
\vspace{-5mm}
\end{table}
We also provide results for the cross-modality recognition task in addition to the cross-modality correlation task depicted in Sec.~5.1 in the main manuscript. Following the experimental setups of CorrNet \cite{Chandar2015CorrelationalNN} and Soft CCA \cite{Chang2017ScalableAE}, for the cross-modality recognition task, we use the MNIST dataset. MNIST \cite{LeCun2005TheMD} consists of a training set of 60000 and a test set of 10000 images of 28$\times$28 hand-written digits. We treat the left and the right halves of the images as the two modalities, and train the common component of \MUCMM\ on the training set. Once the common component is trained, on the test set, the common representations for both modalities (the left and the right halves of the images) are extracted. The common representations for one of the modalities (e.g. left) are used to learn a linear support vector machine (SVM) \cite{Noble2006WhatIA, Pedregosa2011ScikitlearnML} to classify the digits. The model performance is measured by the classification accuracy of the same digits obtained from applying this SVM on the common representations of the other modality (i.e. right). We report the result by doing 5-fold cross-validation on the test set. As one can imagine, the more signals of the common structure the model is able to extract between the two halves of the images, the better the SVM will perform at test time. The results are in Tab.~\ref{mnist_reg}. We expound on our findings next. 

As one can see, \MUCMM\ again outperforms all other methods for all choices of $k$. For the the cross-modality recognition task, the advantage of our formulation is especially clear as the $L_2$ distance objective proposed in Soft CCA \cite{Chang2017ScalableAE} can potentially overshadow the decorrelation regularizations that are to prevent learning redundant information, making the learned common representations less generalizable. 

\subsection{Experimental Setup and Details}
\label{experimental_setup}
This section provides a detailed description of the experimental setups, such as the train/test splits, the chosen network architectures, and the choices of learning rate and optimizer, for the experiments conducted. We describe the architecture of {\MUCMM} in terms of its common component and individual component. We adopt the following abbreviations for some basic network layers
\begin{itemize}
    \item \FL($d_i,d_o,f$) denotes a fully-connected layer with $d_i$ input units, $d_o$ output units, and activation function $f$.
    \item \Conv$\left(c_i, c_o, k_1, f, \batchtwod, O(k_2,s) \right)$ denotes a convolution layer with $c_i$ input channels, $c_o$ output channels, kernel size $k_1$, activation function $f$, the 2D batch norm operation, and the pooling operation $O(k_2,s)$ with another kernel size $k_2$ and stride $s$.
\end{itemize} 

\subsubsection{A Note on the Gradient/Saliency Maps Generation Procedure}
We give a detailed description of how the gradient maps are generated. Given a pair of modalities (of data), we first use the feature extractors to extract feature representations of the pair of data. We note again that we do \emph{not} fine-tune the feature extractor. We next learn the common component (as described in Sec.~3.1) to derive the pair of latent representations (which we will refer to as the common representations) for the common structure. Further, we use the common component, which we do \emph{not} fine-tune once it is trained, to learn the individual component to obtain the pair of representations for the individual structures (which we will refer to as the individual representations). Once we have the common and individual representations,  we calculate the proposed scores detailed in Sec.~3.1 and Sec.~3.2 in the main manuscript for the common and individual structures, respectively. We then use GradCAM~\cite{Adebayo2018SanityCF} along with the computed scores to compute the saliency maps (in the case where the input modalities are images) and visualize those maps. 

\subsubsection{Alternative Formation of Deep CCA}
As stated in the main manuscript, we train the common component of \MUCMM\ on the training set of MNIST and test on the test set. When a hyperparameter-search is needed, we randomly split the training set into 50000 and 10000 images, and use the 50000 images for training and 10000 images for validation/hyperparameters-search.

We use the \texttt{Adam} optimizer with a constant learning rate of 0.0002 for optimization. We train with a batch size of 2048 pairs of halved images for 1000 epochs. We use the following network architecture for the common component of \MUCMM:

\begin{center}
 \label{tab:cbcs_a}
 \begin{tabular}{||c||} 
 \hline
 \textbf{Modality 1}  \\ [0.5ex] 
 \hline\hline
  \FL(392,500,\ReLU) \\ 
 \hline
  \FL(500,300,\ReLU) \\ 
 \hline
  \FL(300,$k$,---) \\
 \hline
  \batchoned($k$)\\
 \hline
\end{tabular}
\quad
 \begin{tabular}{||c||} 
 \hline
 \textbf{Modality 2}  \\ [0.5ex] 
 \hline\hline
  \FL(392,500,\ReLU) \\ 
 \hline
  \FL(500,300,\ReLU) \\ 
 \hline
  \FL(300,$k$,---) \\
 \hline
  \batchoned($k$)\\
 \hline
\end{tabular}
\end{center}
where $k$ denotes the latent dimension in Tab.~\ref{mnist_reg}.

\subsubsection{Object of Interest (OOI) Visualization}
We train \MUCMM\ on the training set of the COCO dataset and report results on the validation set. The training set consists of 117040 images, each of which is associated with multiple labels (that indicate the objects in that image). At each epoch during training, we randomly sample a label from the list of labels associated with each training image (therefore the same image might be trained with different labels at different epochs) and construct the sentence ``\verb|the image contains {label}|'' as the input text.  We use a pre-trained ResNet50 \cite{He2015DeepRL} and a pre-trained BERT \cite{Devlin2019BERTPO} to extract the feature representations for the images and the texts, respectively. For visualizations (after learning the model), we use GradCAM~\cite{Adebayo2018SanityCF} to visualize the last convolutional layer of the ResNet50 using the introduced common score. 

We use the \texttt{Adam} optimizer with a constant learning rate of 0.001 for optimization. We train with a batch size of 2048 pairs of images and texts for 200 epochs. We use the following network architecture for the common component of \MUCMM:
\begin{center}
 \label{tab:cbcs_a}
 \begin{tabular}{||c||} 
 \hline
 \textbf{Image Modality}  \\ [0.5ex] 
 \hline\hline
 Pre-trained Resnet 50\\
  \hline
  \FL(2048,1024,\ReLU) \\ 
 \hline
  \FL(1024,2048,\ReLU) \\ 
 \hline
  \FL(2048,1024,\ReLU) \\
 \hline
  \FL(1024,256,\ReLU) \\
 \hline
   \FL(256,64,---) \\
 \hline
  \batchoned(64)\\
 \hline
\end{tabular}
\quad
 \begin{tabular}{||c||} 
 \hline
 \textbf{Text Modality}  \\ [0.5ex] 
 \hline\hline
  Pre-trained BERT\\
  \hline
  \FL(768,1024,\ReLU) \\ 
 \hline
  \FL(1024,2048,\ReLU) \\ 
 \hline
  \FL(2048,1024,\ReLU) \\
 \hline
  \FL(1024,256,\ReLU) \\
 \hline
   \FL(256,64,---) \\
 \hline
  \batchoned(64)\\
 \hline
\end{tabular}
\end{center}

\subsubsection{Pair of Images}
As described in the main manuscript, we generate each pair of images by first randomly sampling a class from the available 80 fine-grained classes. Given the sampled class, we then randomly sample two images belonging to this class as the pair of images. We generate 60000 such image pairs for training and 5000 for testing. We choose the pre-trained original ResNet50 \cite{He2015DeepRL} and the CLIP ResNets50 \cite{Radford2021LearningTV} as the feature extractors for the paired images. For visualizations (after learning the model), we use GradCAM~\cite{Adebayo2018SanityCF} to visualize the last convolutional layer of the ResNet50 using the introduced common and individual scores. 

We use the \texttt{Adam} optimizer with a constant learning rate of 0.001 for optimization. We train with a batch size of 2048 pairs of images for 200 epochs. We use the following network architecture for the common component of \MUCMM:
\begin{center}
 \label{tab:cbcs_a}
 \begin{tabular}{||c||} 
 \hline
 \textbf{Image Modality 1}  \\ [0.5ex] 
 \hline\hline
  Regu Resnet50\\
  \hline
  \FL(2048,1024,\ReLU) \\ 
 \hline
  \FL(1024,2048,\ReLU) \\ 
 \hline
  \FL(2048,1024,\ReLU) \\
 \hline
  \FL(1024,256,\ReLU) \\
 \hline
   \FL(256,64,---) \\
 \hline
  \batchoned(64)\\
 \hline
\end{tabular}
\quad
 \begin{tabular}{||c||} 
 \hline
 \textbf{Image Modality 2}  \\ [0.5ex] 
 \hline\hline
    CLIP Resnet50\\
  \hline
  \FL(1024,1024,\ReLU) \\ 
 \hline
  \FL(1024,2048,\ReLU) \\ 
 \hline
  \FL(2048,1024,\ReLU) \\
 \hline
  \FL(1024,256,\ReLU) \\
 \hline
   \FL(256,64,---) \\
 \hline
  \batchoned(64)\\
 \hline
\end{tabular}
\end{center}
and the following for the individual component of \MUCMM:
\begin{center}
 \label{tab:cbcs_a}
 \begin{tabular}{||c||} 
 \hline
 \textbf{Image Modality 1}  \\ [0.5ex] 
 \hline\hline
  Regu Resnet50\\
  \hline
  \FL(2048,1024,\ReLU) \\ 
 \hline
  \FL(1024,2048,\ReLU) \\ 
 \hline
  \FL(2048,1024,\ReLU) \\
 \hline
  \FL(1024,256,\ReLU) \\
 \hline
   \FL(256,64,---) \\
 \hline
  \batchoned(64)\\
 \hline
\end{tabular}
\quad
 \begin{tabular}{||c||} 
 \hline
 \textbf{Image Modality 2}  \\ [0.5ex] 
 \hline\hline
    CLIP Resnet50\\
  \hline
  \FL(1024,1024,\ReLU) \\ 
 \hline
  \FL(1024,2048,\ReLU) \\ 
 \hline
  \FL(2048,1024,\ReLU) \\
 \hline
  \FL(1024,256,\ReLU) \\
 \hline
   \FL(256,64,---) \\
 \hline
  \batchoned(64)\\
 \hline
\end{tabular}
\end{center}

\clearpage
\subsection{More Visualizations}
\label{more_plots}
We provide additional visualizations to demonstrate the efficacy of \MUCMM.
\subsubsection{Object of Interest (OOI) Visualization}
We provide two additional visualization examples to demonstrate that by providing \MUCMM\  with a text description of the object of interest (among potentially multiple objects) in an image, \MUCMM\ is able to capture the described object in that image. As one can see, \MUCMM\ is able to locate the object described by the paired text in the image. Furthermore, \MUCMM\ can adapt to different text prompts at \emph{test} time and alter the object of interest highlighted in the same image correspondingly. 

\begin{figure}[!h]
\begin{subfigure}[t]{0.23\textwidth}
    \includegraphics[width=\textwidth]{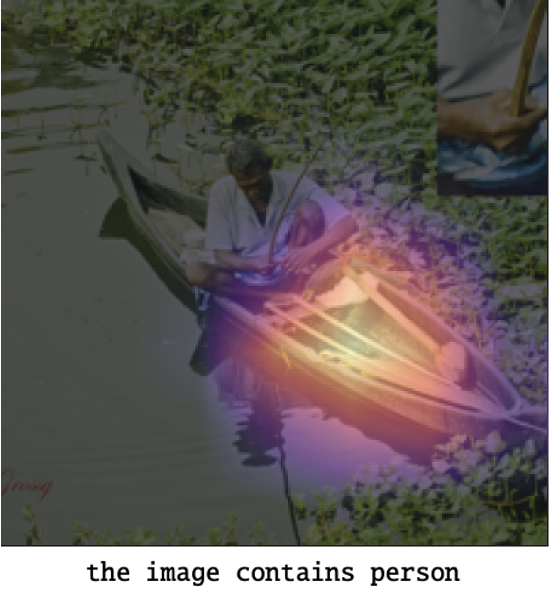}
    \includegraphics[width=\linewidth]{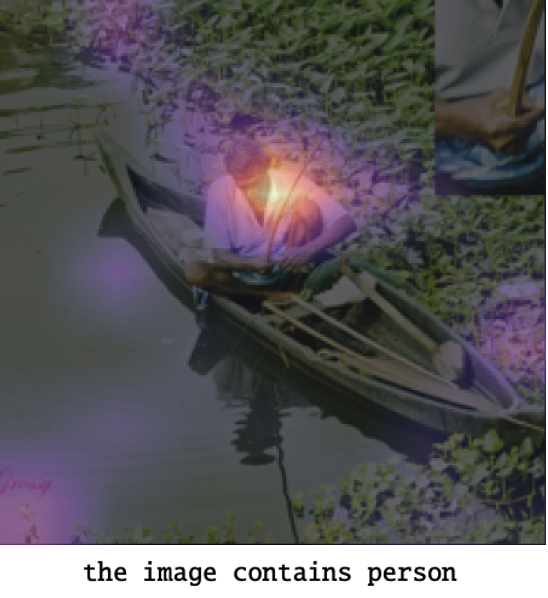}
\end{subfigure}
\smallskip 
\begin{subfigure}[t]{0.23\textwidth}
    \includegraphics[width=\linewidth]{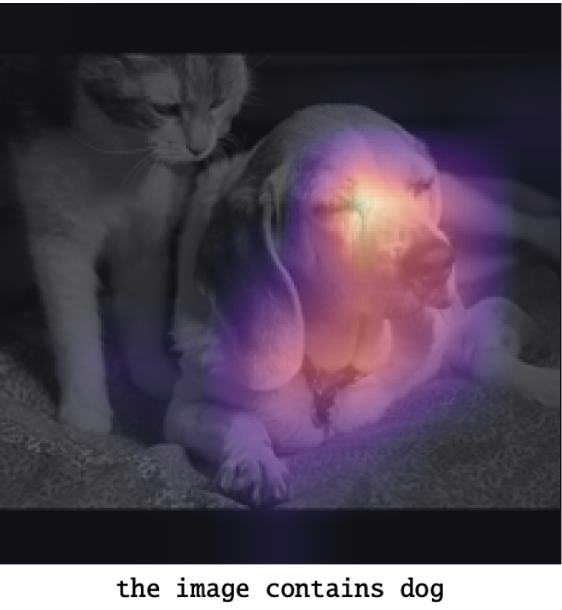}

    \includegraphics[width=\linewidth]{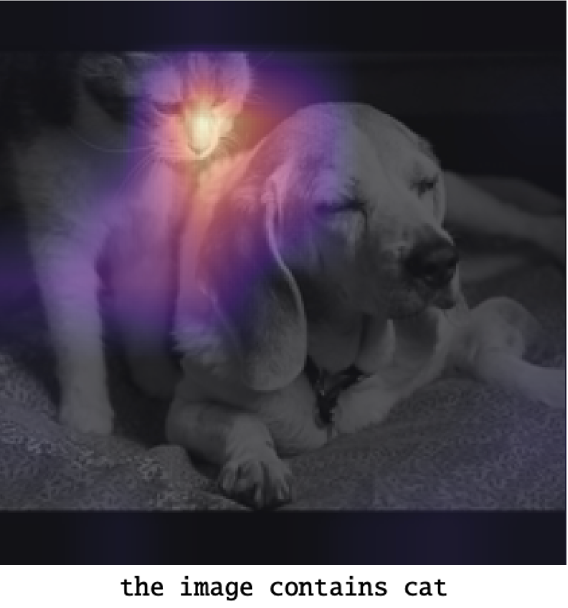}
\end{subfigure}
\smallskip 
\begin{subfigure}[t]{0.23\textwidth}
    \includegraphics[width=\linewidth]{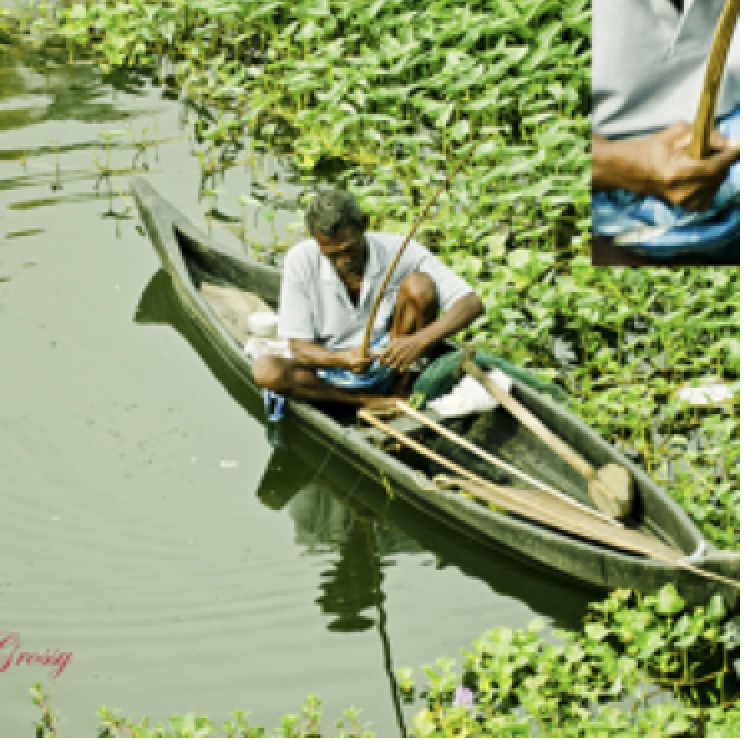}
    \caption{Example 1}
\end{subfigure}\hspace{\fill} 
\begin{subfigure}[t]{0.23\textwidth}
    \includegraphics[width=\linewidth]{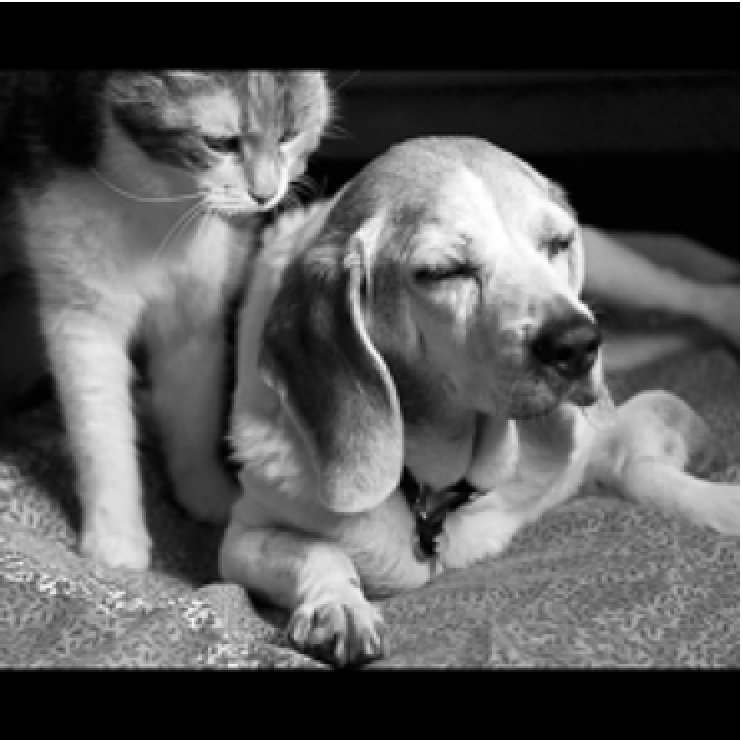}
    \caption{Example 2}
\end{subfigure}
\caption{Two examples (columns) of OOI visualizations at test time are given. The first two rows contain the saliency maps for the example images given different text prompts, where the texts below the images are the prompts. The third row contains the original input images.}
\vspace{-3mm}
\label{coco_roi_append}
\end{figure}

\subsubsection{Pair of Images}
We provide three additional visualization examples to show that \MUCMM\ is able to capture the common and individual structures between two image modalities (please refer to Fig.~\ref{paired_coco_ims_append}). 
\begin{figure*}
\setlength{\tabcolsep}{0.002\textwidth}
\begin{tabular}{@{}cccccc@{}}
\begin{subfigure}[t]{0.16\textwidth}
\includegraphics[width=\textwidth,valign=T]{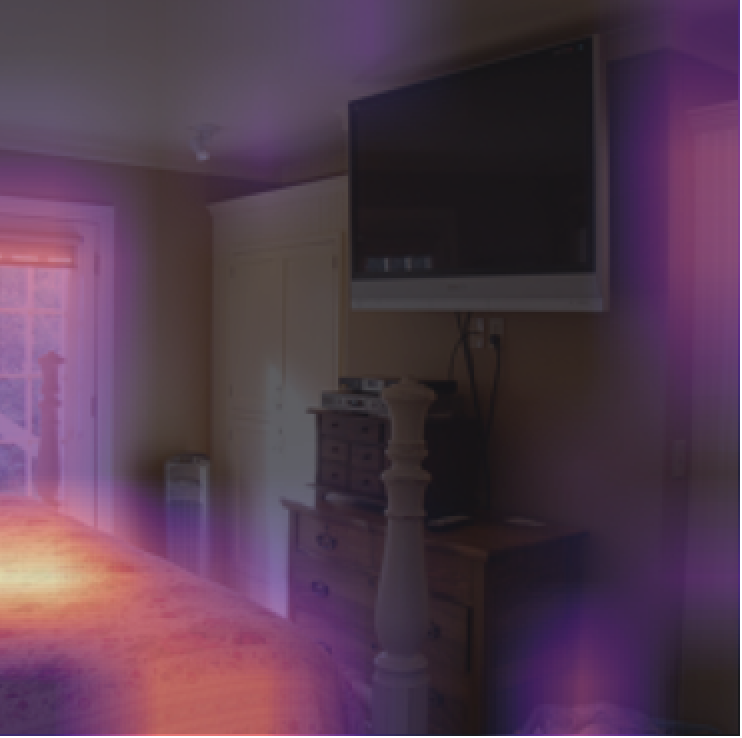}
\end{subfigure} &
\begin{subfigure}[t]{0.16\textwidth}
\includegraphics[width=\textwidth,valign=T]{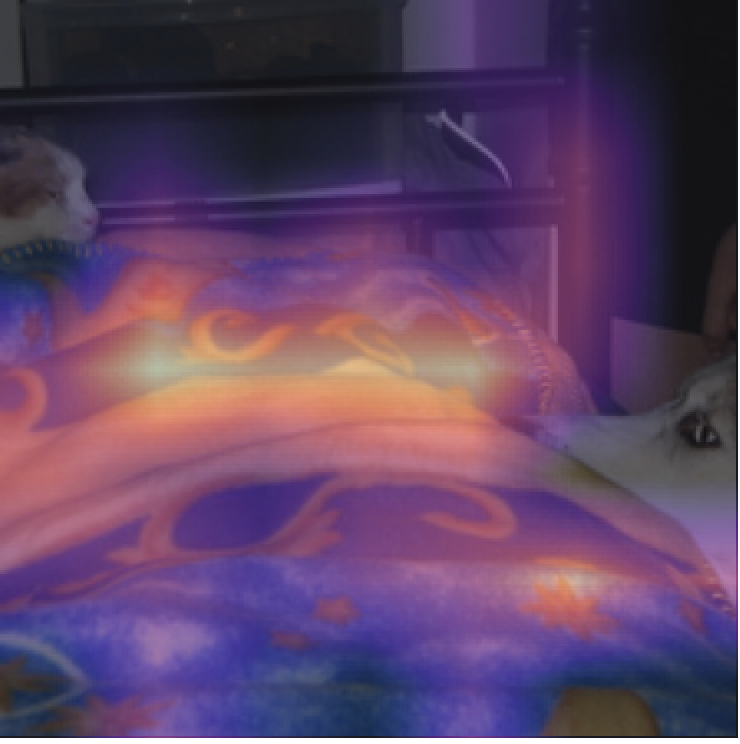}
\end{subfigure} &
\begin{subfigure}[t]{0.16\textwidth}
\includegraphics[width=\textwidth,valign=T]{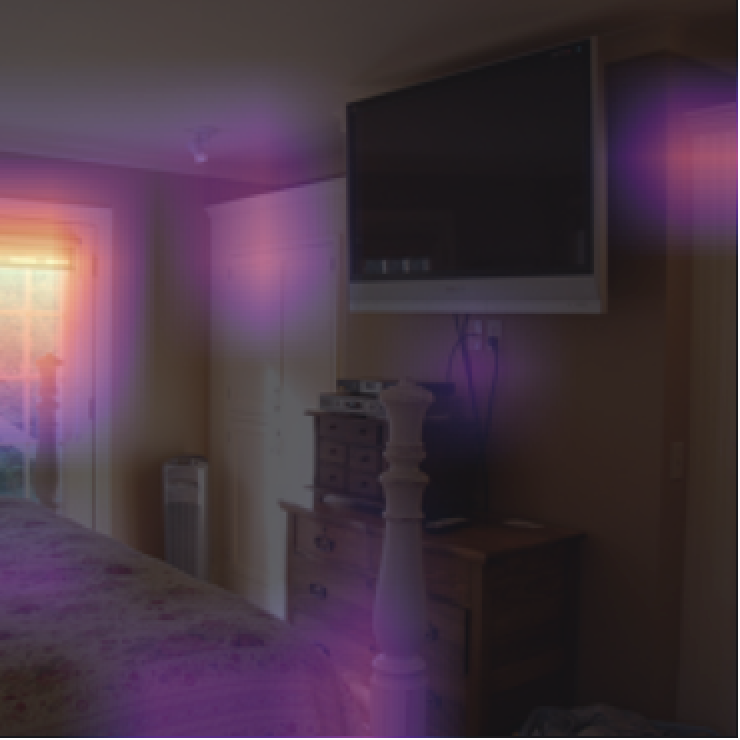}
\end{subfigure}&
\begin{subfigure}[t]{0.16\textwidth}
\includegraphics[width=\textwidth,valign=T]{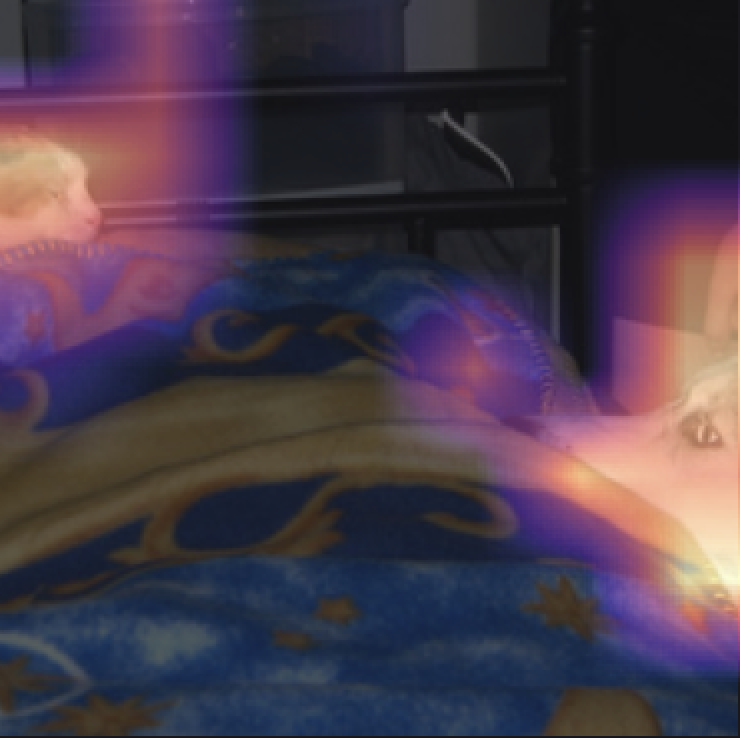}
\end{subfigure}&
\begin{subfigure}[t]{0.16\textwidth}
\includegraphics[width=\textwidth,valign=T]{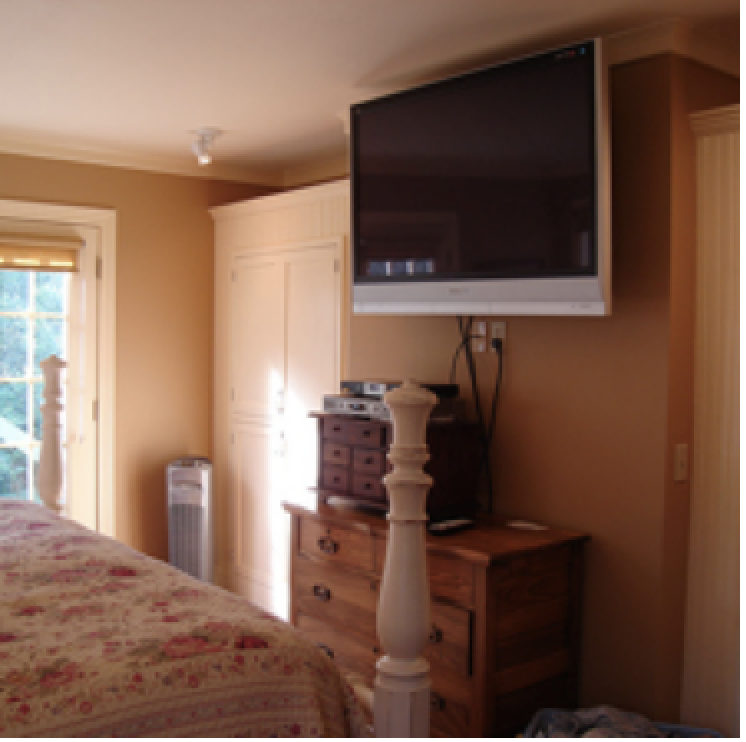}
\end{subfigure}&
\begin{subfigure}[t]{0.16\textwidth}
\includegraphics[width=\textwidth,valign=T]{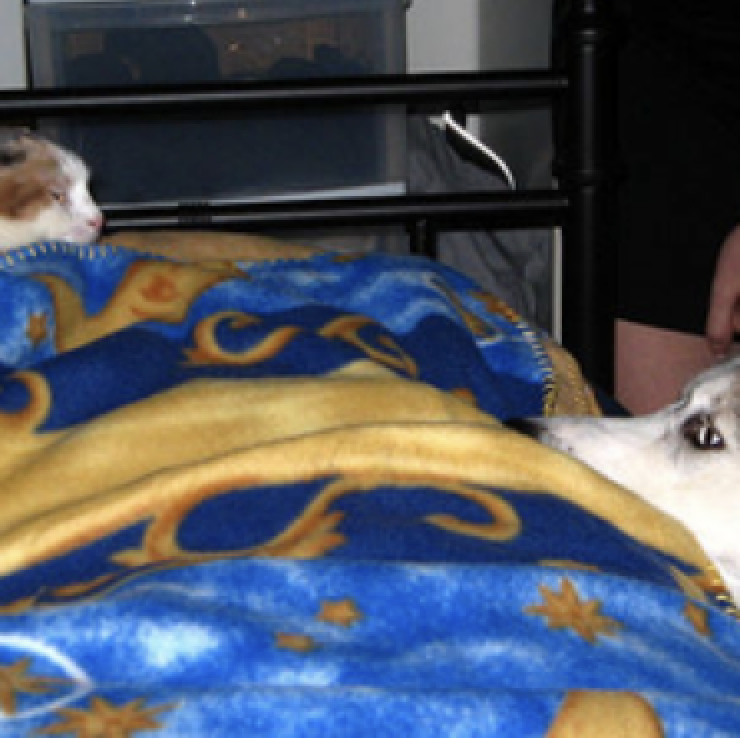}
\end{subfigure}\\
\begin{subfigure}[t]{0.16\textwidth}
\includegraphics[width=\textwidth,valign=T]{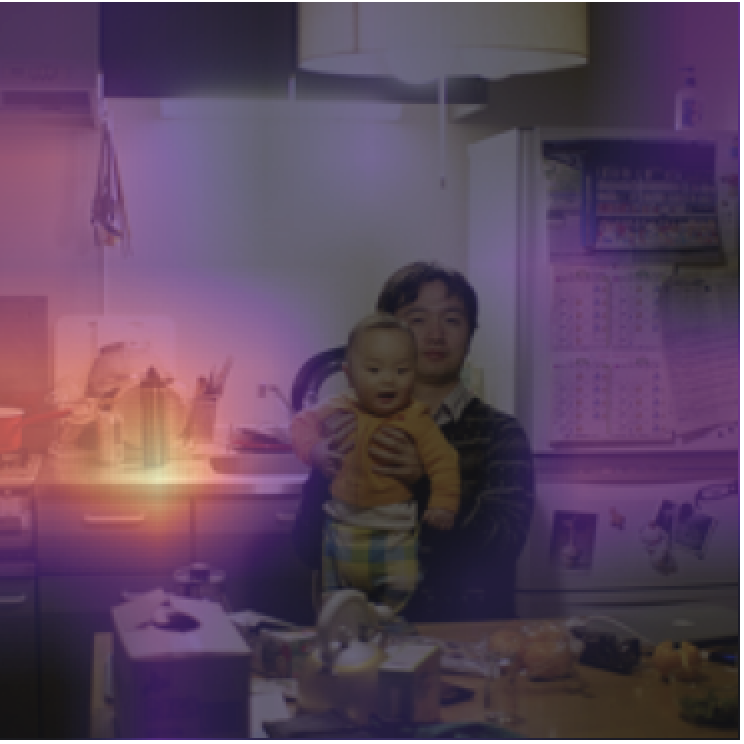}
\end{subfigure} &
\begin{subfigure}[t]{0.16\textwidth}
\includegraphics[width=\textwidth,valign=T]{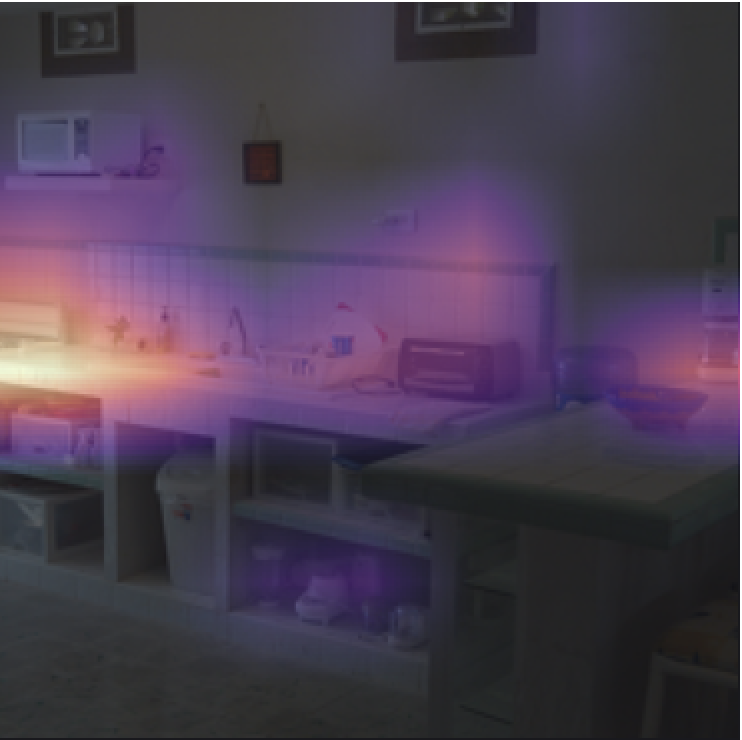}
\end{subfigure} &
\begin{subfigure}[t]{0.16\textwidth}
\includegraphics[width=\textwidth,valign=T]{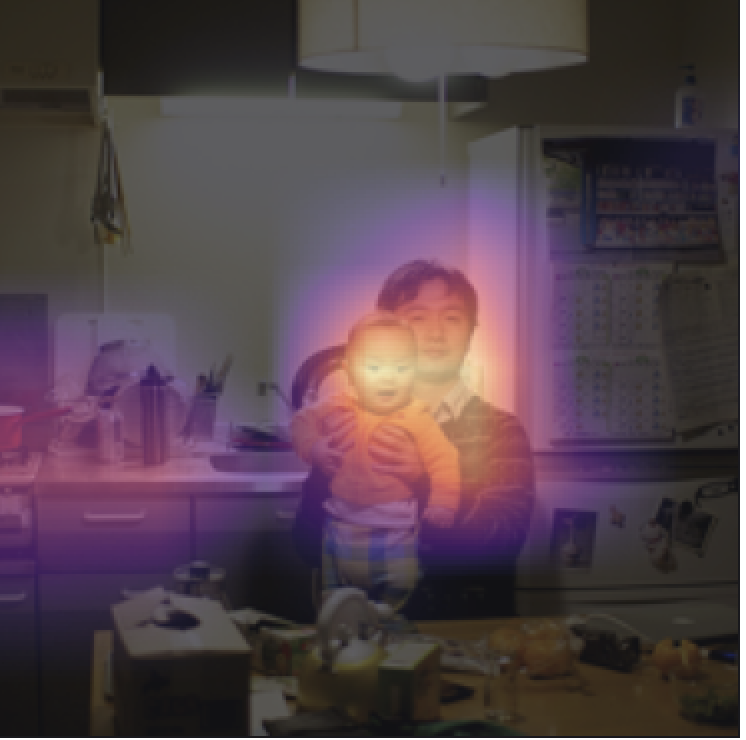}
\end{subfigure}&
\begin{subfigure}[t]{0.16\textwidth}
\includegraphics[width=\textwidth,valign=T]{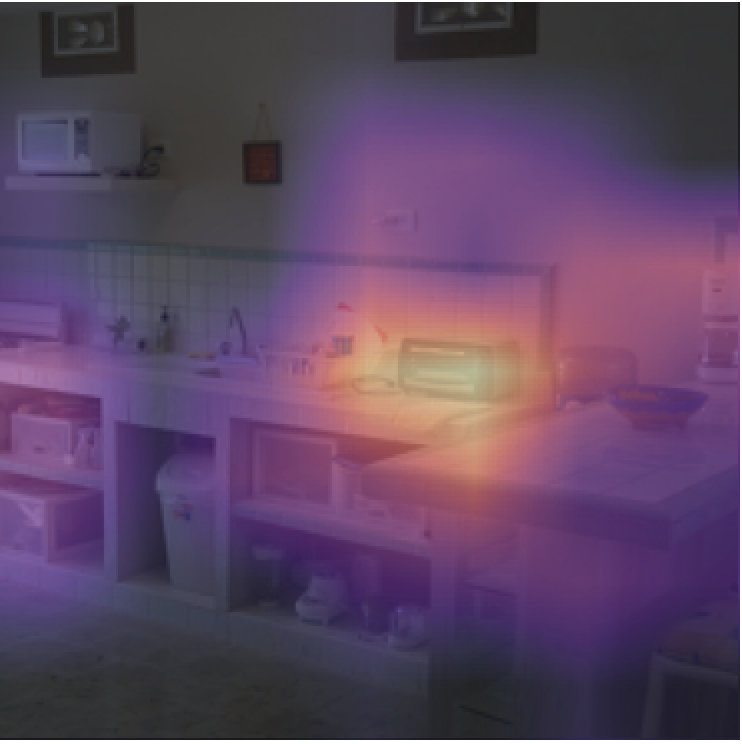}
\end{subfigure}&
\begin{subfigure}[t]{0.16\textwidth}
\includegraphics[width=\textwidth,valign=T]{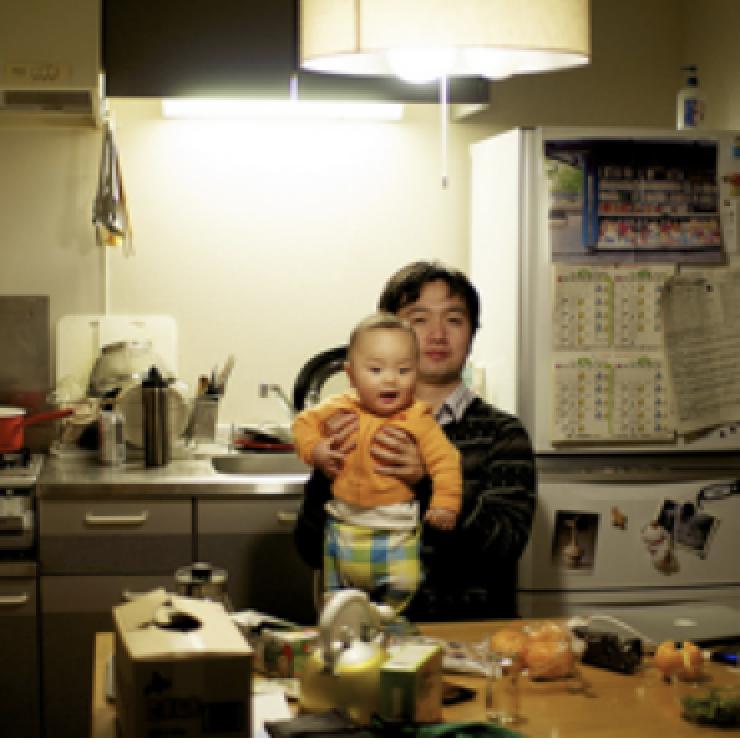}
\end{subfigure}&
\begin{subfigure}[t]{0.16\textwidth}
\includegraphics[width=\textwidth,valign=T]{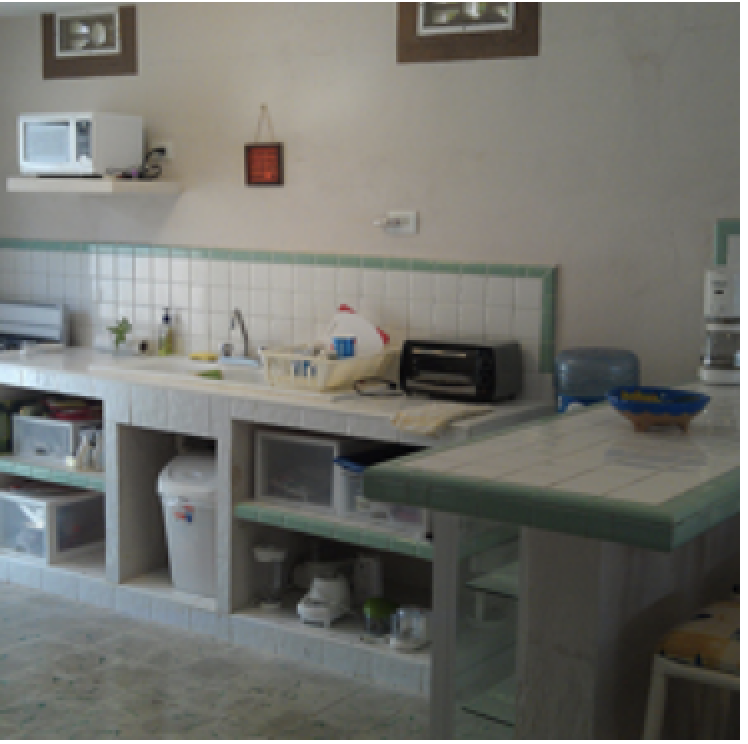}
\end{subfigure}\\
\begin{subfigure}[t]{0.16\textwidth}
\includegraphics[width=\textwidth,valign=T]{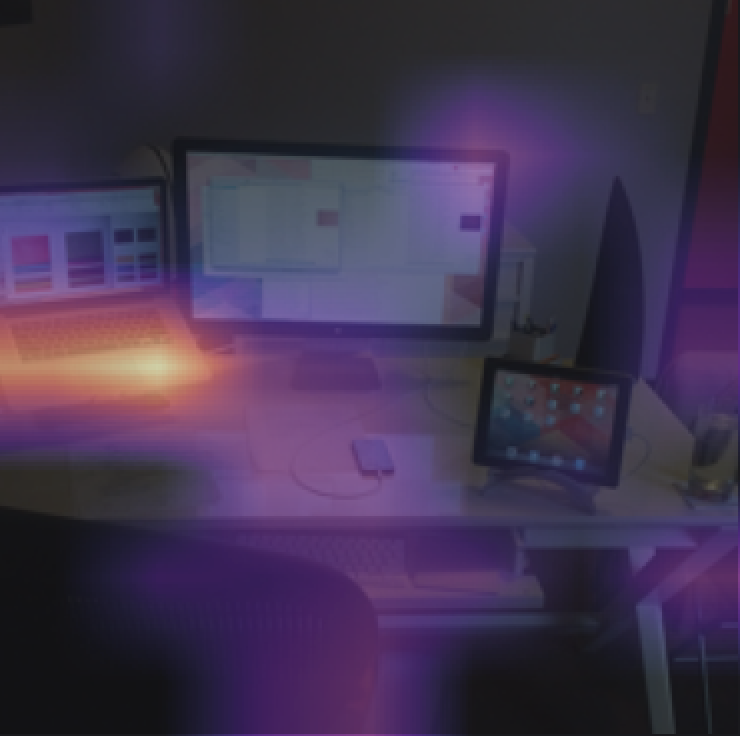}
\end{subfigure} &
\begin{subfigure}[t]{0.16\textwidth}
\includegraphics[width=\textwidth,valign=T]{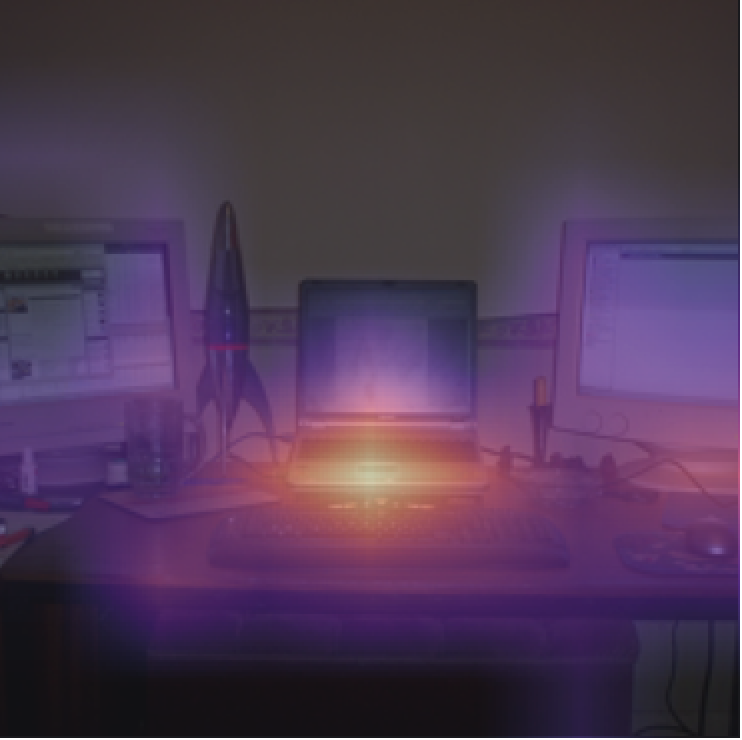}
\end{subfigure} &
\begin{subfigure}[t]{0.16\textwidth}
\includegraphics[width=\textwidth,valign=T]{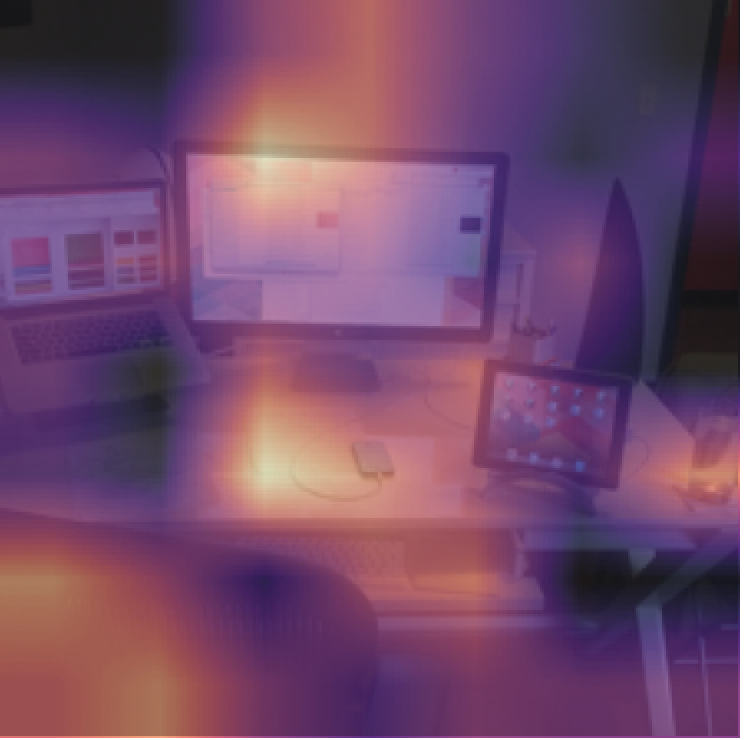}
\end{subfigure}&
\begin{subfigure}[t]{0.16\textwidth}
\includegraphics[width=\textwidth,valign=T]{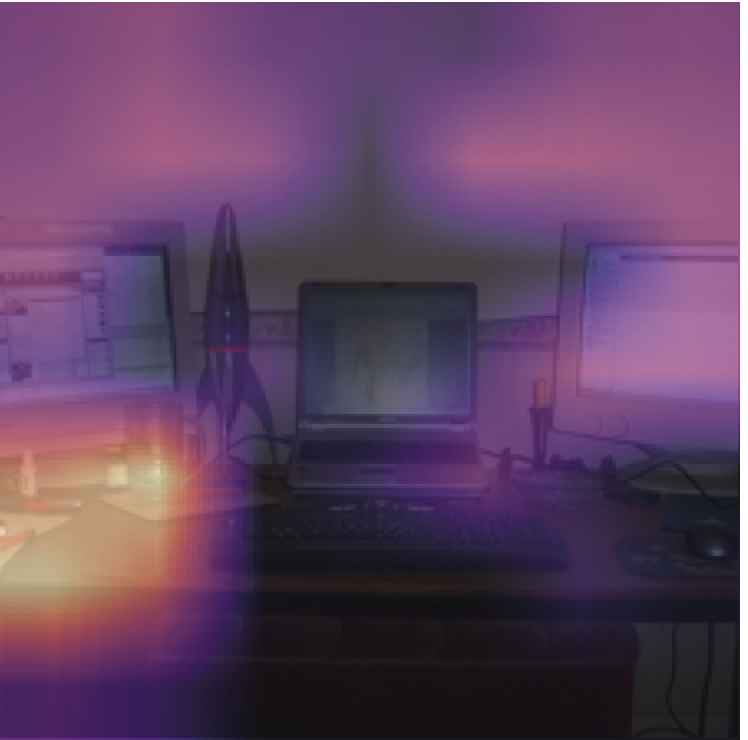}
\end{subfigure}&
\begin{subfigure}[t]{0.16\textwidth}
\includegraphics[width=\textwidth,valign=T]{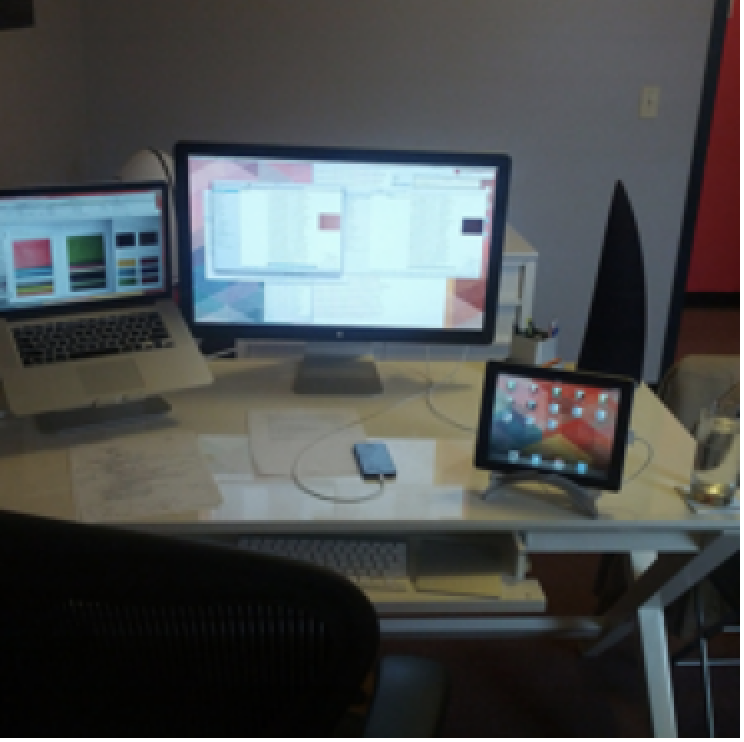}
\end{subfigure}&
\begin{subfigure}[t]{0.16\textwidth}
\includegraphics[width=\textwidth,valign=T]{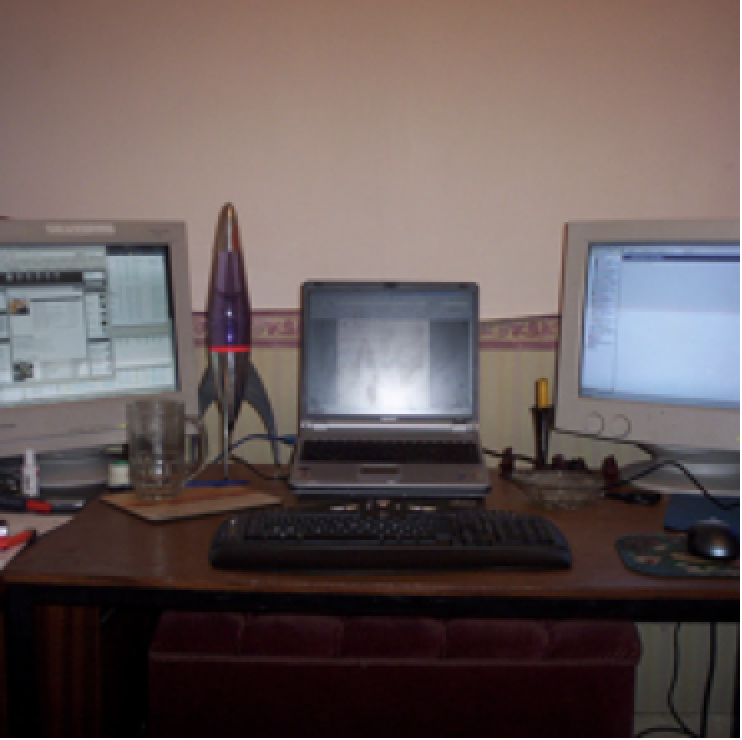}
\end{subfigure}\\
\multicolumn{2}{p{0.3\textwidth}@{}}{\subcaption{Common Structure}}&
\multicolumn{2}{p{0.3\textwidth}@{}}{\subcaption{Individual Structure}}&
\multicolumn{2}{p{0.3\textwidth}@{}}{\subcaption{Original Image}}\\
\end{tabular}

\caption{The results for three sample pairs of images are shown (with each row representing a pair). The Grad-CAM visualizations for the common and individual structures are shown in the first and the middle two columns, respectively. The original pairs of images are shown in the last two columns.}
\label{paired_coco_ims_append}
\end{figure*}

{\small
\bibliographystyle{ieee_fullname}
\bibliography{citation}

\begin{thebibliography}{10}\itemsep=-1pt

\bibitem{Abid2019ContrastiveVA}
Abubakar Abid and James~Y. Zou.
\newblock Contrastive variational autoencoder enhances salient features.
\newblock {\em ArXiv}, abs/1902.04601, 2019.

\bibitem{Adebayo2018SanityCF}
Julius Adebayo, Justin Gilmer, Michael Muelly, Ian~J. Goodfellow, Moritz Hardt,
  and Been Kim.
\newblock Sanity checks for saliency maps.
\newblock In {\em Neural Information Processing Systems}, 2018.

\bibitem{Alqaraawi2020EvaluatingSM}
Ahmed Alqaraawi, M. Schuessler, Philipp Wei{\ss}, Enrico Costanza, and Nadia
  Bianchi-Berthouze.
\newblock Evaluating saliency map explanations for convolutional neural
  networks: a user study.
\newblock {\em Proceedings of the 25th International Conference on Intelligent
  User Interfaces}, 2020.

\bibitem{Andrew2013DeepCC}
Galen Andrew, R. Arora, Jeff~A. Bilmes, and Karen Livescu.
\newblock Deep canonical correlation analysis.
\newblock In {\em International Conference on Machine Learning}, 2013.

\bibitem{Balestriero2022ContrastiveAN}
Randall Balestriero and Yann LeCun.
\newblock Contrastive and non-contrastive self-supervised learning recover
  global and local spectral embedding methods.
\newblock {\em ArXiv}, abs/2205.11508, 2022.

\bibitem{Baltruaitis2017MultimodalML}
Tadas Baltruaitis, Chaitanya Ahuja, and Louis-Philippe Morency.
\newblock Multimodal machine learning: A survey and taxonomy.
\newblock {\em IEEE Transactions on Pattern Analysis and Machine Intelligence},
  41:423--443, 2017.

\bibitem{Bardes2021VICRegVR}
Adrien Bardes, Jean Ponce, and Yann LeCun.
\newblock Vicreg: Variance-invariance-covariance regularization for
  self-supervised learning.
\newblock {\em ArXiv}, abs/2105.04906, 2021.

\bibitem{Bayoudh2021ASO}
Khaled Bayoudh, Raja Knani, Fayçal Hamdaoui, and Abdellatif Mtibaa.
\newblock A survey on deep multimodal learning for computer vision: advances,
  trends, applications, and datasets.
\newblock {\em The Visual Computer}, 38:2939 -- 2970, 2021.

\bibitem{Brigato2020ACL}
Lorenzo Brigato and Luca Iocchi.
\newblock A close look at deep learning with small data.
\newblock {\em 2020 25th International Conference on Pattern Recognition
  (ICPR)}, pages 2490--2497, 2020.

\bibitem{Bugliarello2020MultimodalPU}
Emanuele Bugliarello, Ryan Cotterell, Naoaki Okazaki, and Desmond Elliott.
\newblock Multimodal pretraining unmasked: A meta-analysis and a unified
  framework of vision-and-language berts.
\newblock {\em Transactions of the Association for Computational Linguistics},
  9:978--994, 2020.

\bibitem{Chandar2015CorrelationalNN}
A.~P.~Sarath Chandar, Mitesh~M. Khapra, H. Larochelle, and Balaraman Ravindran.
\newblock Correlational neural networks.
\newblock {\em Neural Computation}, 28:257--285, 2015.

\bibitem{Chang2017ScalableAE}
Xiaobin Chang, Tao Xiang, and Timothy~M. Hospedales.
\newblock Scalable and effective deep cca via soft decorrelation.
\newblock {\em 2018 IEEE/CVF Conference on Computer Vision and Pattern
  Recognition}, pages 1488--1497, 2017.

\bibitem{Devlin2019BERTPO}
Jacob Devlin, Ming-Wei Chang, Kenton Lee, and Kristina Toutanova.
\newblock Bert: Pre-training of deep bidirectional transformers for language
  understanding.
\newblock abs/1810.04805, 2019.

\bibitem{Dhariwal2021DiffusionMB}
Prafulla Dhariwal and Alex Nichol.
\newblock Diffusion models beat gans on image synthesis.
\newblock {\em ArXiv}, abs/2105.05233, 2021.

\bibitem{Feng2017AnglebasedJA}
Qing Feng, Meilei Jiang, Jan Hannig, and J.~S. Marron.
\newblock Angle-based joint and individual variation explained.
\newblock {\em J. Multivar. Anal.}, 166:241--265, 2017.

\bibitem{Hardoon2004CanonicalCA}
David~Roi Hardoon, S{\'a}ndor Szedm{\'a}k, and John Shawe-Taylor.
\newblock Canonical correlation analysis: An overview with application to
  learning methods.
\newblock {\em Neural Computation}, 16:2639--2664, 2004.

\bibitem{He2015DeepRL}
Kaiming He, X. Zhang, Shaoqing Ren, and Jian Sun.
\newblock Deep residual learning for image recognition.
\newblock {\em 2016 IEEE Conference on Computer Vision and Pattern Recognition
  (CVPR)}, pages 770--778, 2015.

\bibitem{Hotelling1936RelationsBT}
Harold Hotelling.
\newblock Relations between two sets of variates.
\newblock {\em Biometrika}, 28:321--377, 1936.

\bibitem{Kingma2013AutoEncodingVB}
Diederik~P. Kingma and Max Welling.
\newblock Auto-encoding variational bayes.
\newblock {\em CoRR}, abs/1312.6114, 2013.

\bibitem{Knapp1978CanonicalCA}
Thomas~R. Knapp.
\newblock Canonical correlation analysis: A general parametric
  significance-testing system.
\newblock {\em Psychological Bulletin}, 85:410--416, 1978.

\bibitem{LeCun2005TheMD}
Yann LeCun and Corinna Cortes.
\newblock The mnist database of handwritten digits.
\newblock 2005.

\bibitem{Lin2014MicrosoftCC}
Tsung-Yi Lin, Michael Maire, Serge~J. Belongie, James Hays, Pietro Perona, Deva
  Ramanan, Piotr Doll{\'a}r, and C.~Lawrence Zitnick.
\newblock Microsoft coco: Common objects in context.
\newblock In {\em European Conference on Computer Vision}, 2014.

\bibitem{Lock2011JOINTAI}
Eric~F. Lock, Katherine~A. Hoadley, J.~S. Marron, and Andrew~B. Nobel.
\newblock Joint and individual variation explained (jive) for integrated
  analysis of multiple data types.
\newblock {\em The annals of applied statistics}, 7 1:523--542, 2011.

\bibitem{Long2022VisionandLanguagePM}
Siqu Long, Feiqi Cao, Soyeon~Caren Han, and Haiqing Yang.
\newblock Vision-and-language pretrained models: A survey.
\newblock In {\em International Joint Conference on Artificial Intelligence},
  2022.

\bibitem{Luo2018ConsistentAS}
Shirui Luo, Changqing Zhang, Wei Zhang, and Xiaochun Cao.
\newblock Consistent and specific multi-view subspace clustering.
\newblock In {\em AAAI Conference on Artificial Intelligence}, 2018.

\bibitem{Noble2006WhatIA}
William~Stafford Noble.
\newblock What is a support vector machine?
\newblock {\em Nature Biotechnology}, 24:1565--1567, 2006.

\bibitem{Paszke2019PyTorchAI}
Adam Paszke, Sam Gross, Francisco Massa, Adam Lerer, James Bradbury, Gregory
  Chanan, Trevor Killeen, Zeming Lin, Natalia Gimelshein, Luca Antiga, Alban
  Desmaison, Andreas K{\"o}pf, Edward Yang, Zach DeVito, Martin Raison, Alykhan
  Tejani, Sasank Chilamkurthy, Benoit Steiner, Lu Fang, Junjie Bai, and Soumith
  Chintala.
\newblock Pytorch: An imperative style, high-performance deep learning library.
\newblock In {\em Neural Information Processing Systems}, 2019.

\bibitem{Pedregosa2011ScikitlearnML}
Fabian Pedregosa, Ga{\"e}l Varoquaux, Alexandre Gramfort, Vincent Michel,
  Bertrand Thirion, Olivier Grisel, Mathieu Blondel, Gilles Louppe, Peter
  Prettenhofer, Ron Weiss, Ron~J. Weiss, J. Vanderplas, Alexandre Passos, David
  Cournapeau, Matthieu Brucher, Matthieu Perrot, and E. Duchesnay.
\newblock Scikit-learn: Machine learning in python.
\newblock {\em J. Mach. Learn. Res.}, 12:2825--2830, 2011.

\bibitem{Radford2021LearningTV}
Alec Radford, Jong~Wook Kim, Chris Hallacy, Aditya Ramesh, Gabriel Goh,
  Sandhini Agarwal, Girish Sastry, Amanda Askell, Pamela Mishkin, Jack Clark,
  Gretchen Krueger, and Ilya Sutskever.
\newblock Learning transferable visual models from natural language
  supervision.
\newblock In {\em International Conference on Machine Learning}, 2021.

\bibitem{Rifai2011ContractiveAE}
Salah Rifai, Pascal Vincent, Xavier Muller, Xavier Glorot, and Yoshua Bengio.
\newblock Contractive auto-encoders: Explicit invariance during feature
  extraction.
\newblock In {\em International Conference on Machine Learning}, 2011.

\bibitem{Simonyan2013DeepIC}
Karen Simonyan, Andrea Vedaldi, and Andrew Zisserman.
\newblock Deep inside convolutional networks: Visualising image classification
  models and saliency maps.
\newblock {\em CoRR}, abs/1312.6034, 2013.

\bibitem{Stefanini2021FromST}
Matteo Stefanini, Marcella Cornia, Lorenzo Baraldi, Silvia Cascianelli,
  Giuseppe Fiameni, and Rita Cucchiara.
\newblock From show to tell: A survey on deep learning-based image captioning.
\newblock {\em IEEE Transactions on Pattern Analysis and Machine Intelligence},
  45:539--559, 2021.

\bibitem{Sun2022AddressingCB}
Qixing Sun, Xiaodong Jia, and Xiaoyuan Jing.
\newblock Addressing contradiction between reconstruction and correlation
  maximization in deep canonical correlation autoencoders.
\newblock In {\em International Conference on Artificial Neural Networks},
  2022.

\bibitem{Tomsett2019SanityCF}
Richard~J. Tomsett, Daniel Harborne, Supriyo Chakraborty, Prudhvi~K. Gurram,
  and Alun~David Preece.
\newblock Sanity checks for saliency metrics.
\newblock {\em ArXiv}, abs/1912.01451, 2019.

\bibitem{Wang2016OnDM}
Weiran Wang, R. Arora, Karen Livescu, and Jeff~A. Bilmes.
\newblock On deep multi-view representation learning: Objectives and
  optimization.
\newblock {\em ArXiv}, abs/1602.01024, 2016.

\bibitem{Wang2015StochasticOF}
Weiran Wang, R. Arora, Karen Livescu, and Nathan Srebro.
\newblock Stochastic optimization for deep cca via nonlinear orthogonal
  iterations.
\newblock {\em 2015 53rd Annual Allerton Conference on Communication, Control,
  and Computing (Allerton)}, pages 688--695, 2015.

\bibitem{Xiao2019MultimodalEA}
Yi Xiao, Felipe Codevilla, Akhil Gurram, Onay Urfalioglu, and Antonio~M.
  L{\'o}pez.
\newblock Multimodal end-to-end autonomous driving.
\newblock {\em IEEE Transactions on Intelligent Transportation Systems},
  23:537--547, 2019.

\bibitem{Xu2021MultiVAELD}
Jie Xu, Yazhou Ren, Huayi Tang, Xiaorong Pu, Xiaofeng Zhu, Ming Zeng, and
  Lifang He.
\newblock Multi-vae: Learning disentangled view-common and view-peculiar visual
  representations for multi-view clustering.
\newblock {\em 2021 IEEE/CVF International Conference on Computer Vision
  (ICCV)}, pages 9214--9223, 2021.

\bibitem{Xu2022MultimodalLW}
Peng Xu, Xiatian Zhu, and David~A. Clifton.
\newblock Multimodal learning with transformers: A survey.
\newblock {\em ArXiv}, abs/2206.06488, 2022.

\bibitem{Zbontar2021BarlowTS}
Jure Zbontar, Li Jing, Ishan Misra, Yann LeCun, and St{\'e}phane Deny.
\newblock Barlow twins: Self-supervised learning via redundancy reduction.
\newblock In {\em International Conference on Machine Learning}, 2021.

\bibitem{Zhang2021FromCC}
Hengrui Zhang, Qitian Wu, Junchi Yan, David~Paul Wipf, and Philip~S. Yu.
\newblock From canonical correlation analysis to self-supervised graph neural
  networks.
\newblock In {\em Neural Information Processing Systems}, 2021.

\bibitem{Zhuang2019ACS}
Fuzhen Zhuang, Zhiyuan Qi, Keyu Duan, Dongbo Xi, Yongchun Zhu, Hengshu Zhu, Hui
  Xiong, and Qing He.
\newblock A comprehensive survey on transfer learning.
\newblock {\em Proceedings of the IEEE}, 109:43--76, 2019.

\bibitem{Zou2005RegularizationAV}
Hui Zou and Trevor~J. Hastie.
\newblock Regularization and variable selection via the elastic net.
\newblock {\em Journal of the Royal Statistical Society: Series B (Statistical
  Methodology)}, 67, 2005.

\end{thebibliography}
}

\end{document}